\documentclass[a4paper,fleqn]{cas-dc}

\usepackage[numbers]{natbib}

\def\tsc#1{\csdef{#1}{\textsc{\lowercase{#1}}\xspace}}
\tsc{WGM}
\tsc{QE}
\tsc{EP}
\tsc{PMS}
\tsc{BEC}
\tsc{DE}

\usepackage{adjustbox}
\graphicspath{ {./images/} }
\usepackage{algorithm}
\usepackage[noend]{algpseudocode}
\usepackage{wrapfig}
\usepackage{amsthm}
\usepackage{hyperref}       
\usepackage{url}            
\usepackage{xcolor}         

\newtheorem{theorem}{Theorem}
\newtheorem{lemma}[theorem]{Lemma}

\usepackage{multirow}

\usepackage{placeins}
\usepackage{enumitem}

\algblockdefx{FORP}{ENDFP}[1]%
  {\textbf{for }#1 \textbf{do in parallel}}%
 
\algblockdefx{ForEce}{EndForEce}[1]%
  {\textbf{for }#1 \textbf{do}}%
  {\textbf{end for}}

\begin{document}

\let\WriteBookmarks\relax
\def\floatpagepagefraction{1}
\def\textpagefraction{.001}

\shorttitle{Attack-Resistant Federated Learning}

\shortauthors{Isik-Polat et~al.}

\title [mode = title]{ARFED: Attack-Resistant Federated Averaging Based on Outlier Elimination}                      



%
\author[1]{Ece Isik-Polat}[orcid=0000-0002-0728-5390]

\cormark[1]


\ead{eceisik@metu.edu.tr}



\affiliation[1]{organization={Graduate School of Informatics, Middle East Technical University},
    addressline={Universiteler, Dumlupinar Blv. 1/6 D:133, Cankaya},
    city={Ankara},
    postcode={06800},
    country={Turkey}}

\author[1]{Gorkem Polat}[orcid=0000-0002-1499-3491]
\ead{gorkem.polat@metu.edu.tr}

\author[1]{Altan Kocyigit}[orcid=0000-0001-5003-4127]
\ead{kocyigit@metu.edu.tr}

\cortext[cor1]{Corresponding author}

\fntext[fn1]{You can visit \url{https://github.com/eceisik/ARFED} to see all
implemented methods and designed experiments.}

\fntext[fn2]{The numerical calculations reported in this paper were partially performed using TUBITAK ULAKBIM, High Performance and Grid Computing Center (TRUBA resources).}



\begin{abstract}
In federated learning, each participant trains its local model with its own data and a global model is formed at a trusted server by aggregating model updates coming from these participants. Since the server has no effect and visibility on the training procedure of the participants to ensure privacy, the global model becomes vulnerable to attacks such as data poisoning and model poisoning. Although many defense algorithms have recently been proposed to address these attacks, they often make strong assumptions that do not agree with the nature of federated learning, such as assuming Non-IID datasets. Moreover, they mostly lack comprehensive experimental analyses. In this work, we propose a defense algorithm called ARFED that does not make any assumptions about data distribution, update similarity of participants, or the ratio of the malicious participants. ARFED mainly considers the outlier status of participant updates for each layer of the model architecture based on the distance to the global model. Hence, the participants that do not have any outlier layer are involved in model aggregation. We have performed extensive experiments on diverse scenarios and shown that the proposed approach provides a robust defense against different attacks. To test the defense capability of the ARFED in different conditions, we considered label flipping, Byzantine, and partial knowledge attacks for both IID and Non-IID settings in our experimental evaluations. Moreover, we proposed a new attack, called organized partial knowledge attack, where malicious participants use their training statistics collaboratively to define a common poisoned model. We have shown that organized partial knowledge attacks are more effective than independent attacks.
\end{abstract}


\begin{highlights}
\item We propose an attack-resistant federated averaging method called ARFED. ARFED does not make any unrealistic assumptions about data distributions, update similarities of participants or having information about the malicious participant ratios.

\item We conducted comprehensive experiments to evaluate the performance of ARFED on three different datasets under different attack types, organized/independent attacks, and IID/Non-IID data distributions. We also compared the performance of ARFED with recently proposed defense algorithms.

\item We propose an adaptive attack which is an organized variation of partial knowledge attack in \cite{fang2020local}. In this attack type, participants use their training statistics collaboratively to define a common poisoned model.

\item We showed that although recently proposed robust defense algorithms perform well on IID settings, they provide a little or no defense in Non-IID settings. In contrast, the proposed approach ARFED stabilizes the convergence and protects the global model from various attacks in any setting.

\end{highlights}


\maketitle

\section{Introduction}\label{introduction}

Digitalization has been reshaping the economy, organizations, and society \cite{altan_a_osmundsen2018digital}. It is driving the extensive use of artificial intelligence techniques to increase efficiency and productivity, lower costs, and improve the quality of products and services in all industries and sectors \cite{altan_b_9202502}. Recent advances in computing and communications technologies, widespread deployment of smart and connected devices, and the proliferation of cloud computing enable the collection and cost-effective storage and processing of massive data that is essentially the fuel for successful organizations today. Hence, there has been a great interest in the Internet of Things (IoT) and Big Data concepts \cite{altan_c_SESTINO2020102173}. In this context, machine learning is one of the primary techniques to extract non-obvious and useful patterns and actionable insight from data. Specifically, deep learning offers enormous potential and has achieved remarkable success \cite{altan_d_6817512}. Machine learning in general, and deep learning in particular, is a computation-intensive task that processes data usually collected from many sources and stored in a central location. However, with big data, collecting, storing, and processing data in a scalable and efficient manner is a fundamental challenge \cite{altan_e_wang2020big}. Besides, the performance and adequacy of a trained model mostly rely on the availability of sufficiently large data relevant to the task of interest. Hence, data collected from various sources such as online transaction systems, IoT devices, smartphones, and social media are integrated to have a rich dataset to derive high-quality models. Such data generally contain sensitive information of which collection and use may cause violation of regulations such as the General Data Protection Regulation (GDPR) \cite{altan_f_8622621, custers2019eu, voigt2017eu}. Hence, data privacy turns out to be a fundamental challenge.

Federated Learning (FL) \cite{mcmahan2017communication} is a distributed approach to training a machine learning model without requiring training data to be available in a central place. In FL, participants with relevant data and processing resources collaborate to train a machine learning model without revealing their data. In a typical FL setting, a trusted server communicates with the participants to jointly train a common model in several iterations. The server chooses a model architecture, determines training parameters, and initializes a global model, which is iteratively improved by performing local training on participants' devices. In each iteration, the server sends the global model to the participants, which use their training data to improve the model for a while and send the resulting local models back to the server. Then the server aggregates received models to have a better global model. These iterations can be continued until some convergence criteria are met. FL is a viable approach to overcome data privacy issues as the participants do not need to disclose their training data; instead, they only share locally trained model parameters. Hence, FL is a promising approach to large-scale application of machine learning in domains where protecting user's privacy is of great concern, such as healthcare \cite{altan_g_SINGH2022380}, smart city \cite{altan_h_zheng2022applications}, and others widely employing IoT and big data technologies \cite{altan_i_gadekallu2021federated}.  

FL has some unique characteristics \cite{konevcny2016federated}. Training data may be massively distributed onto a large number of devices with heterogeneous resources, and the sizes and distributions of data on different devices may vary considerably. FL aims to train a single model that can perform well on all participants' data. However, this may not be possible when participants have heterogeneous data and processing resources \cite{altan_j_9174890}. Not independent and identically distributed (IID) data on devices lead to severe issues in FL, and improving the performance of FL on non-IID data is an active research area \cite{altan_k_MA2022244}. This statistical heterogeneity affects the convergence behavior of FL and may lead to biased models toward the participants having larger training data. The heterogeneity of storage, computation, and communication resources of participants, expensive communication (especially when there are a massive number of participants), and ensuring privacy against inference attacks are other core challenges in FL \cite{altan_l_9084352}. The presence of adversaries manipulating their data or locally trained models exacerbates the problems caused by the heterogeneity of data and resources \cite{altan_j_9174890}. Therefore, FL is also vulnerable to security attacks as the central server has no control over the participants' data and local training processes. There are several vulnerabilities that an attacker might exploit to manipulate the learning process, manipulate the global model, or gain access to participants' private data \cite{altan_m_MOTHUKURI2021619}. For instance, an attacker may pretend like an ordinary participant(s) or the central server or gain control of one or more participants or the central server to target the learning process without getting noticed. The attacker's goal may be to slow down or impede the convergence of model training, degrade the trained model's performance, manipulate the global model to get wrong inferences under specific cases, lead to an ineffective global model, or extract participants' local data from the parameters exchanged during training rounds. Data poisoning \cite{tolpegin2020data} and model poisoning \cite{fang2020local} are two significant security threats that attackers can pose in FL. In data poisoning attacks, malicious participants manipulate their training data by adding noise or flipping target labels. In model poisoning attacks, participants alter their models before sending them to the server.

Several aggregation approaches and optimization algorithms, mainly variations of the gradient descent algorithm, have been proposed for model training with FL \cite{konevcny2016federated, duan2019astraea, NEURIPS2020_4ebd440d, yuan2020federated}. Federated Averaging (FedAvg) \cite{mcmahan2017communication} is one of the most commonly used FL algorithms. In each iteration, FedAvg aggregates the locally trained models returned by the participants to form the new global model by averaging. Each parameter of the new global model is set to the weighted arithmetic mean of the corresponding parameters of the participants' models. In this process, weights are determined according to the number of training examples in the participants to give each training example an equal weight on the global model update. This feature represents a vital vulnerability if there are malicious participants sending arbitrary or deliberately manipulated model parameters, which are likely to be very different from the updates received from reliable participants. Hence, checking the models before aggregation to identify anomalies and dropping (or lowering the weights of) the models coming from suspicious participants in the aggregation could be a promising approach to deal with poisoning attacks.

In this paper, we propose the Attack-Resistant Federated Learning (ARFED) algorithm, which is an extended version of FedAvg. The primary objective of the algorithm is to defend against poisoning attacks in FL. In ARFED, the parameters of the models received from the participants in each iteration are examined using a statistical outlier detection technique to identify potentially malicious participants. Accordingly, such participants are discarded in the model aggregation step to mitigate poisoning threats. Many defense methods for poisoning attacks are proposed in the literature, as summarized in Section 2. These methods usually require some knowledge about the attacks, such as malicious participant ratio, examining local datasets, which may compromise the privacy of participants, assuming IID data, which is not valid in typical FL settings or introducing too much computational overhead, which restricts the practical implementations. Unlike the other defense methods, ARFED employs a relatively simple malicious participant identification technique that does not require making unrealistic assumptions inconsistent with typical FL settings.  

The main contributions of this paper are as follows:
\begin{enumerate}[label=(\alph*)]
    \item We propose an extension to FedAvg called ARFED to defend against poisoning attacks. Unlike the other similar algorithms proposed in the literature, ARFED does not make unrealistic assumptions about data distributions or participants' update similarities. Nor does it require information about attacks, such as the malicious participant ratios, in advance. As shown in Section 3, the computational complexity of the extension made to FedAvg is lower compared to other similar defense algorithms.
    \item We evaluated the performance of the vanilla FedAvg, ARFED, and two similar defense approaches in various FL scenarios, such as IID data, non-IID data, and under various kinds of organized and independent poisoning attacks, as well as the no-attack case. The results show that attacks in non-IID cases are more severe than in IID cases. Moreover, the attacks committed by an organized group of attackers can be more detrimental than those implemented by a group of independent attackers.
    \item The experimental results show that ARFED can mitigate the effects of independent and organized attacks in IID and non-IID data cases. It outperforms the evaluated alternatives, especially in non-IID data and organized attack scenarios. Moreover, it does not cause significant performance loss under no-attack cases where the evaluated alternative defense methods cause some performance loss, especially in non-IID cases. Hence, ARFED can be used to defend against various kinds of poisoning attacks without worrying about significant performance degradation under no-attack cases.
\end{enumerate}

The rest of the paper is organized as follows. The related work in the literature is reviewed in Section \ref{related_work}. Section \ref{proposed_method} describes the proposed outlier-detection-based malicious participant identification technique and delineates the ARFED algorithm. The details of the experimental design used to evaluate the performance and compare it with the other approaches are presented in Section \ref{experimental_design}. The experimental results are presented and discussed in Section \ref{experimental_results}. Section \ref{conclusion} gives concluding remarks and directions for future work. Finally, the supplementary experimental results are given in Appendix \ref{appendix}.

\section{Related Work}\label{related_work}

Security is a critical issue in FL as it is vulnerable to several attacks, such as poisoning, backdoor, free-riding, inference, and eavesdropping \cite{altan_m_MOTHUKURI2021619}. This paper focuses on poisoning attacks in FL and proposes a defense mechanism called ARFED. There are two kinds of poisoning attacks: data poisoning \cite{tolpegin2020data} and model poisoning \cite{fang2020local}. In data poisoning attacks, malicious participants manipulate or modify data, for instance, by adding noise to the training data or label flipping \cite{bhagoji2019analyzing, sattler2020byzantine, fung2018mitigating, shen2016auror}. In model poisoning attacks, participants alter the models sent to the server in each iteration. Byzantine attack in which malicious participants send arbitrary updates is one of the prevalent model poisoning attacks \cite{fang2020local, bhagoji2019analyzing, sattler2020byzantine, blanchard2017machine, pmlr-v80-mhamdi18a}. Backdoor attacks aim to affect the global model adversarially on a particular sub-task, for example, by making the model classify "trucks" as "planes" by adding small visual artifacts to the training set \cite{sattler2020byzantine, fung2018mitigating, bagdasaryan2020backdoor, kairouz2019advances, NEURIPS2020_b8ffa41d_yesyoucan}.

In FL, the central server does not have access to the participants' training data or control over the participants' training process. Therefore, aggregation carried out by the server is the most appropriate step to defend against such security attacks. Hence, many defense algorithms incorporated into aggregation rules are proposed to handle these attacks and prevent performance degradation caused by them \cite{blanchard2017machine, bagdasaryan2020backdoor, chen2018draco, xie2019zeno, xie2020zeno++, yin2018byzantine}. Apart from the defense mechanisms incorporated into the aggregation process carried out by the central server, there are also decentralized solutions. With the distributed ledger technology provided by the blockchain, the need for a trusted central server can be removed \cite{altan_n_zhang2020blockchain}. Using blockchain technologies in FL can also provide robustness against adversarial attacks \cite{altan_o_9134967}. Although this is one of the promising approaches to defend against attacks, it is out of the scope of this paper.

Many studies have shown that the defense mechanisms proposed in the literature usually make assumptions that do not hold for practical FL settings \cite{fang2020local, kairouz2019advances, NEURIPS2019_ec1c5914, NEURIPS2019_415185ea}. In particular, Non-IID datasets and organized (coordinated) attacks bring severe issues to the learning problem and invalidate assumptions of previous works. Moreover, defense strategies that require examining the local datasets and utilizing partial or complete knowledge of the training process (defense against backdoor attacks and approaches using data sanitization) are not appropriate in practical FL settings. Thus, analyzing and developing these approaches in realistic FL environments is an important study area.

Attack-robust FL has been a heavily studied topic in recent years. Yin et al. \cite{yin2018byzantine} introduced two different approaches instead of solely averaging gradients. The first method was the coordinate-wise median and the second was the coordinate-wise trimmed mean that excludes the highest and smallest values with the given percentage. Blanchard et al. \cite{blanchard2017machine} proposed a Byzantine fault-tolerant SGD algorithm called Krum that combines the majority-based and square-distance methods. El Mhamdi et al. \cite{pmlr-v80-mhamdi18a} introduced a method that combines Krum and trimmed mean, called Bulyan. These methods presume IID data and the ratio of malicious participants should be known in each communication round, which usually does not hold in FL settings.

There are also clustering or similarity metrics based methods that work under certain conditions and with certain assumptions. Fung et al. \cite{fung2018mitigating} assume that the trusted participants have a unique distribution and as a result, their gradient updates vary. Since malicious participants have a common goal, their gradient updates tend to be more similar. Based on this assumption, Fung et al. proposed the FoolsGold algorithm that identifies participants who make similar gradient updates with a method based on cosine similarity and reduces the learning rates of these participants. Sattler et al. \cite{sattler2020byzantine} proposed a method that clusters the participants based on the pairwise cosine dissimilarities between their updates and considers the elements of the largest cluster as benign. Tolpegin et al. \cite{tolpegin2020data} presented a method based on identifying malicious participant clusters with a visualization that is obtained by applying Principal Component Analysis to the parameters of the last layer of the participants' local models. Unlike Fung et al. \cite{fung2018mitigating}, Sattler et al. \cite{sattler2020byzantine} and Tolpegin et al. \cite{tolpegin2020data} worked with benign participants that have similar updates on IID data.

The data distribution and update similarities of participants are two essential factors that should be examined in detail. Most of the recent studies proposed methods for IID case \cite{sattler2020byzantine, blanchard2017machine, pmlr-v80-mhamdi18a, yin2018byzantine, tolpegin2020data}; however, Non-IID distribution of participants' data is one of the key properties of FL and it was emphasized that existing techniques for Byzantine tolerant distributed learning do not perform well when data of participants are Non-IID \cite{mcmahan2017communication, bagdasaryan2020backdoor, kairouz2019advances}. Although the proposed method in \cite{fung2018mitigating} addresses Non-IID data distribution, it only covers a very specific case where trusted participants have unique updates and malicious participants have similar updates.

Although there have been notable new studies proposing aggregation methods for distributed learning that ensure the convergence of the global model, they sacrificed classification performance in exchange for convergence, resulting in ineffective strategies that are not useful for FL settings \cite{bhagoji2019analyzing, blanchard2017machine, pmlr-v80-mhamdi18a, yin2018byzantine}.

\section{Attack-Resistant Federated Learning}\label{proposed_method}

The Attack-Resistant Federated Learning (ARFED) algorithm is based on the Federated Averaging (FedAvg) algorithm \cite{mcmahan2017communication}, which is one of the most widely used aggregation algorithms in FL \cite{Li2020On}. In FedAvg, a server initializes a global parametric model such as a multi-layer neural network which the participants collectively train in several rounds. In each round, the server sends the current global model to the participants, which apply the gradient descent algorithm to train the current global model using their locally available data for several epochs and then send the resulting local models back to the server. The server aggregates the participants' locally trained models by calculating the weighted averages of corresponding parameters in the local models and updates the global model accordingly. The weights are determined according to the number of training examples in the participants.  

\begin{figure}[hbt!]
  \centering
  \includegraphics[width=\columnwidth]{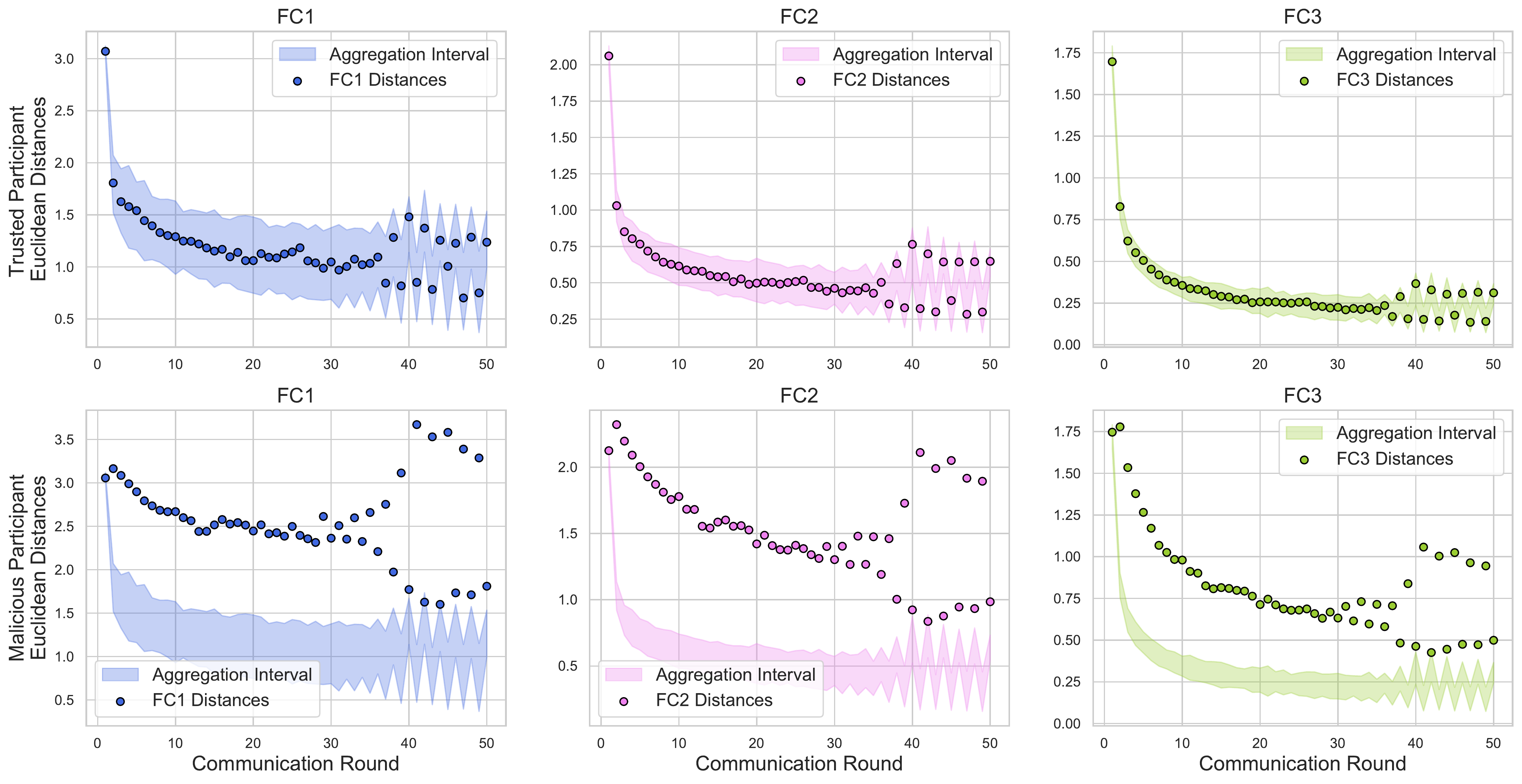}
  \caption{The layer-wise distances of a randomly selected trusted participant (top row) and a randomly selected malicious participant (bottom row) to the global model under label flipping attack on the MNIST dataset. The points refer to the distances of the participants to the same layer of the global model in corresponding rounds. The shaded regions indicate the upper and lower bounds determined for each round according to the $IQR$ outlier identification technique.}
  \label{fig:mnist_malicious_distances}
\end{figure}

When there is no attack and the global model converges, the latest local models received from the participants are unlikely to be far from the global model. In contrast, if there are malicious participants, their models drift apart from the trusted ones and the global model \cite{bhagoji2019analyzing}. Hence, outlier identification can be a promising approach to defend against model poisoning and data poisoning attacks by identifying and filtering out potentially malicious participants. To illustrate the situation and gain insight into the malicious participant detection problem, we carried out an experiment using the experimental setup introduced in Section \ref{experimental_design}. In this experiment, a set of trusted and malicious participants collaborate to train a three-layer neural network on the MNIST dataset. In order to quantify how far the local models are from the global model, we considered the differences between the parameters of the global model at the beginning of a round and the corresponding parameters of local models at the end of that round. We used Euclidean distances between the global and the local model parameters to facilitate comparisons and outlier detection. Furthermore, to improve granularity, we calculated distances for each layer separately. To this end, each layer having $K$ parameters (i.e., weights and biases) is represented by a point in $K$-dimensional space, and the Euclidean distances between the points corresponding to the global model's layer and the participants' models are computed for each layer. We considered the distribution of distances for each participant for each layer and used Inter Quartile Range ($IQR$) method to identify outliers. IQR is a measure of the spread of data and defined as the $25^{th}$ percentile (i.e., $Q_1$, the first quartile) and $75^{th}$ percentile (i.e., $Q_3$, the third quartile) of the data. That is, $IQR=Q_3-Q_1$. According to the $IQR$ technique, values less than $Q_1-1.5 \times IQR$ (i.e., the lower limit) or greater than $Q_3+1.5 \times IQR$ (i.e., the upper limit) are considered outliers. Figure~\ref{fig:mnist_malicious_distances} shows the layer-wise distances of randomly selected trusted and malicious participants' models distances to the global model in successive federated learning rounds of the experiment. The shaded regions in the plots indicate $[Q_1-1.5 \times IQR, Q_3+1.5 \times IQR]$ range for non-outlier values through rounds. As seen from the plots, for this illustrative case, the selected trusted participant's distances to the global model almost always remain within the non-outlier range for all layers. On the other hand, the selected malicious participant's distances to the global models are almost always outliers. Inspired by this experiment, we incorporated the IQR-based outlier identification technique to eliminate model updates from potentially malicious participants in the aggregation step of FedAvg.

The pseudocodes of the procedures carried by an ARFED server and ARFED participants are presented in Algorithm~\ref{alg:defensealgorithm}. Each participant executes the ParticipantUpdate procedure (Lines 01-05), which is invoked by the server, and the server executes the ServerUpdate procedure (Lines 06-29). The notation used in this algorithm is introduced in Table~\ref{tab:notation_table}. For the sake of simplicity, we presumed that the algorithm is run for a predefined number of rounds, all participants are involved in all training rounds, and participants employ batch gradient descent using a predefined learning rate in their local training process. However, this algorithm can easily be extended to realize other practices, such as repeating the process until convergence criteria are met, involving a subset of participants in each round, employing mini-batch gradient descent in local training, and learning rate scheduling.

\begin{algorithm}[ht!]
\caption{ARFED.}
\label{alg:defensealgorithm}
\begin{algorithmic}[1]
\Procedure{ParticipantUpdate}{$p, \textbf{W}$}
    \ForEce{$ e=1, 2,..., E$ }   
        \State $\textbf{W} \gets \textbf{W} -\eta \frac{1}{n_p}
      \sum_{i=1}^{n_p}
      \nabla_{\textbf{W}}\mathcal{L}(x_{p,i}, y_{p,i}, \textbf{W}) $
    \EndForEce
    \State return $\textbf{W}$
\EndProcedure

\Procedure{ServerUpdate}{}
\State $\textbf{W}^0 \gets$ initial weights(random)
\ForEce{$t=1,2,..., T$ }
    \ForEce{$p=1,2,..., P$ } 
        \State $\textbf{W}_p^t \gets$ \Call{ParticipantUpdate}{$p, \textbf{W}^{t-1}$}
    \EndForEce
    \State $r^t \gets \{p|p=1,..., P\}$
    \ForEce{$l=1,2,..., L$}   
        \State $d_l^t = \{\}$
        \ForEce{$p=1,2,..., P$} 
            \State $d_{p,l}^t \gets \|flatten(\textbf{W}^{t-1}[l])-\textbf{W}_p^t[l]\|_2$  
            \State $d_l^t \gets d_l^t + [d_{p,l}^t] $
            \Comment{append to the list}
        \EndForEce
        \State $min\_d_l^t= Q_1(d_l^t) - 1.5 \times IQR(d_l^t)$
        \State $max\_d_l^t= Q_3(d_l^t) + 1.5 \times IQR(d_l^t)$
        
        \ForEce{\textbf{each} $p$ \textbf{in} $r^t$}
            \If{ $d_{p,l}^t<min\_d_l^t$ \textbf{or} $d_{p,l}^t>max\_d_l^t$}
            \State $r^t \gets r^t-\{p\}$
            \Comment{remove from the set}
            \EndIf
            \State \textbf{end if}
        \EndForEce 
    \EndForEce
    \State $\textbf{W}^t \gets \frac{1}{\sum_{p \in r^t}n_p}
    \sum_{p \in r^r} n_p \times \textbf{W}_p^t$
\EndForEce
\State \textbf{return $W^t$}
\EndProcedure
\end{algorithmic}
\end{algorithm}

\begin{table}[ht!]
\caption{Notation Table.}
\label{tab:notation_table}
\begin{tabular}{p{0.25\columnwidth}  p{0.65\columnwidth}} 
\hline
$P$  & Number of participants   \\
$p$  & Participant $p$, $p \in \{1, ... , P\}$\\
$n_p$  & Number of training examples in participant $p$ \\
$x_{p,i}$  & Features of $i^{th}$ training example in participant $p$   \\
$y_{p,i}$  &  Label of $i^{th}$ training example in participant $p$   \\
$T$  & Number of training rounds   \\
$t$  & Federated learning round $t$, $t \in \{1, ..., T\}$   \\
$L$  & Number of layers in the global model  \\
$l$  & Layer $l$ in the global model, $l \in \{1, ..., L\}$   \\
$K_l$ & Number of parameters (weights and biases) in layer $l$  \\
$\textbf{W}^t$ & Weights of the global model at round $t$  \\
$\textbf{W}^t[l]$ & Weights of the $l^{th}$ layer of the global model at round $t$ \\
$\textbf{W}^t_p$ & Weights of the local model in participant $p$ at round $t$\\
$\textbf{W}^t_p[l]$ & Weights of the $l^{th}$ layer of the local model in participant $p$ at round $t$ \\
$\mathcal{L}(x, y, \textbf{W})$ &  Loss function for a training example $x$ with label $y$ on a model with weights $\textbf{W}$ \\
$\nabla_\textbf{W}\mathcal{L}(x, y, \textbf{W})$& Gradient of the loss function with respect to weights \textbf{W}\\
$E$ & Number of local iterations (epochs) in each round  \\
$e$ & Local training iteration $e$, $e \in \{1, ..., E\}$  \\
$\eta$ & Learning rate  \\
$w$=flatten($\textbf{W}$) & flatten a tensor $\textbf{W}$ to a vector $w$\\
$\|w\|_2$ & $\ell^2$-norm of a vector  \\
$Q_q(d)$  & $q^{th}$ quartile of values in list $d$ for $q \in \{1, 2, 3 ,4\}$\\
$IQR$  & Inter Quartile Range of values in list $d$; $IQR(d)= Q_3(d)-Q_1(d)$ \\
$r^t$ & Set of participants marked as reliable at round $t$\\
$d_{p,l}^t$  &Distance between the global model’s $l^{th}$ layer and participant $p$’s $l^{th}$ layer in round $t$  \\
$d_l^t$  & List of distances between the global model’s $l^{th}$ layer and all participants $l^{th}$ layers in round $t$\\
$min\_d_l^t$  & Lower distance threshold for layer $l$ in round $t$ to mark a participant reliable \\
$max\_d_l^t$  & Upper distance threshold for layer $l$ in round $t$ to mark a participant reliable  \\
\hline
\end{tabular}
\end{table}

ARFED is essentially an extension to FedAvg to train a parametric machine learning model such as a multi-layer neural network using the Gradient Descent algorithm. In ARFED, a server trains a randomly initialized global model (Line 07) in $T$ rounds by collaborating with $P$ participants (Lines 08-26). In each round $t$, participants receive the previous round’s global model represented by parameters $\textbf{W}^{t-1}$ from the server (Lines 09-11). Then they apply the gradient descent algorithm to improve the model using their locally available training data for $E$ epochs according to a server-defined loss function $\mathcal{L}(x, y, \textbf{W})$ (Lines 02-04). They finally return the resulting local model to the server (Line 05). The participant update in ARFED is essentially the same as the participant update of the vanilla FedAvg. Like FedAvg, the server initializes a global model (Line 07) and trains the model for $T$ rounds (Lines 08-28) by involving participants and returns the resulting model (Line 29). The main difference between FedAvg and ARFED is in the aggregation process. In FedAvg, the local models ($W_p^t$) received from the participants are aggregated to update the global model by computing weighted averages for each parameter (like Line 27, but including all participants). However, in ARFED, only the models received from participants deemed reliable (i.e., the participants in the reliable participant set $r^t$) in that round are included in the model aggregation step (Line 27). In the beginning, all participants are assumed to be reliable (Line 12), so $r^t$ includes all participants. In order to identify the potentially malicious participants, for each layer $l$, a list of distances $d_l^t$ between the received participant models and the global model is computed (Lines 15-18). Accordingly, a lower distance threshold $min\_d_l^t$ and an upper distance threshold $max\_d_l^t$ are determined by computing $Q_1$, $Q_2$, and $IQR$ of distance values in $d_l^t$ (Lines 19-20). Then, if a participant $p$ is an outlier in a layer $l$ (i.e., its distance $d_{p,l}^t$ is less than the lower distance threshold $min\_d_l^t$ or greater than the upper distance threshold $max\_d_l^t$ it is considered malicious and removed from the reliable participant set $r^t$ (Lines 21-25). Hence, if a participant is identified as potentially malicious according to at least one layer, it is identified as not reliable. After evaluating all layers, the federated averaging is applied to the local models of the participants deemed reliable (i.e., the participants in $r^t$) (Line 27).

Defense strategies such as trimmed mean and coordinate-wise median rely on including a model's parameters partially. Each parameter within the model is evaluated individually; some participants' updates are included, and the rest are discarded in the aggregation step for each parameter. Hence, a different group of participants can potentially contribute to each parameter. The primary motivation of the proposed all-or-nothing approach is that each participant is evaluated in a holistic approach. Parameters of a neural network are highly dependent on each other; therefore, independently evaluating each parameter may lead to misleading inferences. If any layer of a model update is an outlier, it is a sign of a malicious participant; therefore, it is not reasonable to include that participant in the aggregation step. In the proposed approach, for a participant's model to be included in the calculation of global model aggregation, each layer must fall within the safe interval calculated for that layer, i.e., a consensus should be ensured among all layers of the local model. Interestingly, experiments show that the ratio of unreliable participants determined in the proposed approach is very close to actual malicious participant ratios (see Figures \ref{fig:cifar_byz_num_of_outliers_fig} and \ref{fig:cifar_lf_num_of_outliers_fig} in the Appendix~\ref{appendix}).


\begin{lemma}
\label{lemma:complexity}
For an $L$-layer neural network with $K$ parameters (weights and biases) in each layer collectively trained by $P$ participants, the time complexity of reliable participant identification in ARFED is $\mathcal{O}(L\cdot P\cdot(K+log{P}))$. 
\end{lemma}

\begin{proof}	
The server computes the differences and $\ell^2$-norms of $K$-dimensional vectors for each layer with $K$ parameters, $\mathcal{O}(K)$. As there are $L$ layers in the models received from $P$ participants, the distance calculation is $\mathcal{O}(L\cdot P\cdot K)$. For each layer, the server sorts the distances of $P$ participants to find $Q_1$, $Q_3$, and $IQR$, determining lower and upper thresholds, which is $\mathcal{O}(P\cdot \log{P})$. As there are $L$ layers, the threshold determination is $\mathcal{O}(L\cdot P\cdot \log{P})$. Finally, for each layer, the distances of $P$ participants are checked if they are outliers or not, which is $\mathcal{O}(P)$. As there are $L$ layers, the outlier detection is $\mathcal{O}(L\cdot P)$. As a result, the algorithm’s time complexity can be found as
$\mathcal{O}(L\cdot P\cdot K+L\cdot P\cdot \log{P}+L\cdot P)\rightarrow
\mathcal{O}(L\cdot P\cdot(K+\log{P}+1))\rightarrow
\mathcal{O}(L\cdot P\cdot(K+\log{P}))$.

The computational complexity is an essential factor in the practical use of an algorithm. Lemma~\ref{lemma:complexity} states that the time complexity of reliable participant identification (i.e., the extension made to FedAvg)  in ARFED is $\mathcal{O}(L\cdot P\cdot(K+log{P}))$. Hence it is much more efficient than that of Krum and its variant Bulyan, which are quadratic, $\mathcal{O}(K\cdot P^2)$. ARFED is slightly more efficient than the coordinate-wise median and trimmed mean as they require sorting for all individual parameters (ARFED only makes sorting as many as the number of layers).
\end{proof}


\section{Experimental Design}\label{experimental_design}

The dimensions of the experimental designs are the datasets, the data distribution of the participants, attack types, attacker types, and baseline method selection.

\subsection{Datasets}\label{datasets}

We conducted experiments on MNIST \cite{lecun1998mnist}, CIFAR10 \cite{krizhevsky2009learning}, and Fashion-MNIST \cite{xiao2017fashion} datasets which are widely used by researchers to evaluate FL approaches \cite{tolpegin2020data, bhagoji2019analyzing, sattler2020byzantine, blanchard2017machine, pmlr-v80-mhamdi18a, bagdasaryan2020backdoor, yin2018byzantine}. MNIST dataset contains 28$\times$28 grayscale images with 50,000 training images and 10,000 testing images. CIFAR10 dataset contains 32$\times$32 color images with 50,000 training images and 10,000 testing images, and Fashion-MNIST contains 28$\times$28 grayscale images with 60,000 training images and 10,000 testing images.

There are 100 participants in all experiments. There are two model architectures for MNIST, and only one architecture for CIFAR10 and Fashion-MNIST. The details of the model architectures, and the hyperparameters are given in Appendix~\ref{architecture}. For CIFAR10  experiments, data augmentation techniques such as horizontal flipping and random cropping, and training strategies like learning rate scheduling and gradient clipping were applied to enhance the model performance.


\subsection{Data Distributions}\label{datadistributions}

One of the dimensions of our experiments is the data distribution of the participants, which can be either IID or Non-IID. For IID cases, training datasets are distributed to the participants randomly and uniformly, i.e., each participant has each class equally. On the other hand, in the Non-IID case, each participant has examples of only two randomly selected classes for MNIST and Fashion-MNIST and examples of only five randomly selected classes for CIFAR10.

\subsection{Attack and Attacker Types}\label{attack_and_attacker_types}

Another dimension is attack types. Three attack types are examined: label flipping attacks, Byzantine attacks, and adaptive partial knowledge attacks. In label flipping attacks, malicious nodes flip their ground truth labels with a target class label. In Byzantine attacks, malicious participants send random weight updates from a standard normal distribution with zero mean and unit standard deviation. In adaptive attacks, malicious participants use statistics of local models' parameter to manipulate sending weights.

Lastly, attacker types are investigated. Independent attackers are malicious participants incapable of coordinating with each other, acting individually, and sending random updates to the server. Organized or coordinated attackers are malicious participants that carry out the attack in an organized or coordinated manner and send similar updates to the server. For example, in independent label flipping attacks, malicious participants flip their ground truth labels with an arbitrary target label, e.g., if there are two malicious participants with label 7 in their data sets, one flips 7 to 1, while the other flips to 4. On the other hand, in organized label flipping attacks, the malicious participants flip ground truth labels with consistent target labels, e.g., all malicious participants that have 7 in their datasets flip the label as 1. 

In order to increase the success of the attacks and reduce the likelihood of malicious participants being caught, the replaced classes were chosen as semantically similar as possible. The replaced classes in the organized setting for each data set are presented in Table~\ref{table:organized_lf_classes}.

\begin{table}[hbt!]
\centering
\caption{Replaced classes for organized label flipping attack.}
\label{table:organized_lf_classes}
\begin{adjustbox}{width=\columnwidth,center}
\begin{tabular}{cccccccc}
\multicolumn{2}{c}{\textbf{MNIST}} & \textbf{~} & \multicolumn{2}{c}{\textbf{Fashion-MNIST}} & \textbf{~} & \multicolumn{2}{c}{\textbf{CIFAR10}}  \\ 
\hline\hline
Original & Replaced                & ~          & Original    & Replaced                     & ~          & Original & Replaced                    \\ 
\hline
0        & 9                       &            & T-shirt/Top & ~Shirt                       &            & Plane    & ~Bird~                      \\
1        & 7                       &            & Trouser     & ~Dress                       &            & ~Car     & Truck                       \\
2        & 5                       &            & Pullover    & Coat                         &            & ~Bird~   & Plane                       \\
3        & 8                       &            & ~Dress      & Trouser                      &            & Cat~     & Dog~                        \\
4        & 6                       &            & Coat        & Pullover                     &            & Deer     & Horse                       \\
5        & 2                       &            & Sandal      & Sneaker                      &            & Dog~     & Cat~                        \\
6        & 4                       &            & ~Shirt      & T-shirt/Top                  &            & Frog     & Ship                        \\
7        & 1                       &            & Sneaker     & Ankle Boot                   &            & Horse    & Deer                        \\
8        & 3                       &            & Bag         & Sandal                       &            & Ship     & Frog                        \\
9        & 0                       &            & Ankle Boot  & Sneaker                      &            & Truck    & ~Car   \\
\hline
\end{tabular}
\end{adjustbox}
\end{table}

Similarly, malicious participants send different random weights for independent Byzantine attacks while they send the same random weights in the organized setting. The details of independent and organized adaptive partial knowledge attacks are given in Section~\ref{partial_attack}  

\subsection{Adaptive Partial Knowledge Attack}\label{partial_attack}

In the original partial knowledge attack in \cite{fang2020local}, the malicious participants train their local models with their local data. Then, for each parameter,  mean, $\mu_w$, and standard deviation,  $\sigma_w$, are estimated among the malicious participants' parameters. Later, each malicious participant determines the update changing direction, $s_w$, for each parameter by looking at the global model they have received at the beginning of the FL round (if $w^{t+1}_m >= w^t \rightarrow s_w=1$, else $s_w=-1$).  

If $s_w=-1$, each malicious participant replaces the parameter with a number uniformly sampled from the interval $[\mu_w + 3\sigma_w, \mu_w + 4\sigma_w ]$. If $s_w=1$, the malicious participant replaces the parameter with a number uniformly sampled from the interval $[\mu_w -4\sigma_w, \mu_w -3\sigma_w ]$.  We show the results of the original version of this attack under the independent experimental setting (See Table~\ref{tab:fang_iid_table} and Table~\ref{tab:fang_noniid_table}).


In our experimental setting, the malicious participants send the same parameters to the server in the “Organized Byzantine” attacks. Based on this idea, we adopted the attack in \cite{fang2020local} for the "Organized" version. For this time, to decide the direction of change ($s_w$), the parameter of the global model is compared with the mean parameter, $\mu_w$, for once instead of comparing separately for each participant’s parameter. Then, the same ($s_w$) is used for each participant.
If $s_w=-1$, the malicious participants replace the parameters with the same number sampled from the interval $[\mu_w + 3\sigma_w, \mu_w + 4\sigma_w ]$. If $s_w=1$, the malicious participants replace the parameter with the same number sampled from the interval  $[\mu_w -4\sigma_w, \mu_w -3\sigma_w]$.  

   
   

\subsection{Baseline Method Selection}\label{baseline_method}

Fang et al. \cite{fang2020local} have shown that trimmed mean and coordinate-wise median are more robust to attacks than Krum and its variant Bulyan. Unlike the coordinate-wise median, trimmed mean requires the knowledge of malicious participant ratio, which does not meet the nature of a realistic FL setting, Yet, both coordinate-wise median (will be referred to as CwMedian in the rest of the article) and trimmed mean (will be referred to as TrimmedMean in the rest of the article) are agnostic to update similarity of participants unlike \cite{sattler2020byzantine, fung2018mitigating, tolpegin2020data},  and they are outlier based methods as the proposed method ARFED; therefore, they are chosen as the baseline methods in our performance evaluations.

\section{Experimental Results \& Discussion}\label{experimental_results}

Attacks usually compromise convergence as well as the performance of the models and cause oscillations in test set accuracies; therefore, reporting only the score of the last communication round can be misleading because the peak point or lowest point of this oscillation may occur randomly. Thus, the minimum and the maximum accuracies achieved on the test set in the last ten rounds are reported in tables to show the severity of the oscillations created by attacks. In all experiments, \textbf{NoDefense} refers to the vanilla FedAvg, \textbf{CwMedian} refers to the coordinate-wise median and \textbf{TrimmedMean} refers to the trimmed mean.

The tables show the results of all datasets, while the figures show only the MNIST-2NN experiments in this section. The corresponding figures of other datasets and architectures can be found in the \ref{experiments}. Moreover, additional information about box plot factor comparison is presented in \ref{boxplotcomparison}.

\subsection{No Attack}

\begin{figure}[ht!]
  \centering
  \includegraphics[width=\columnwidth]{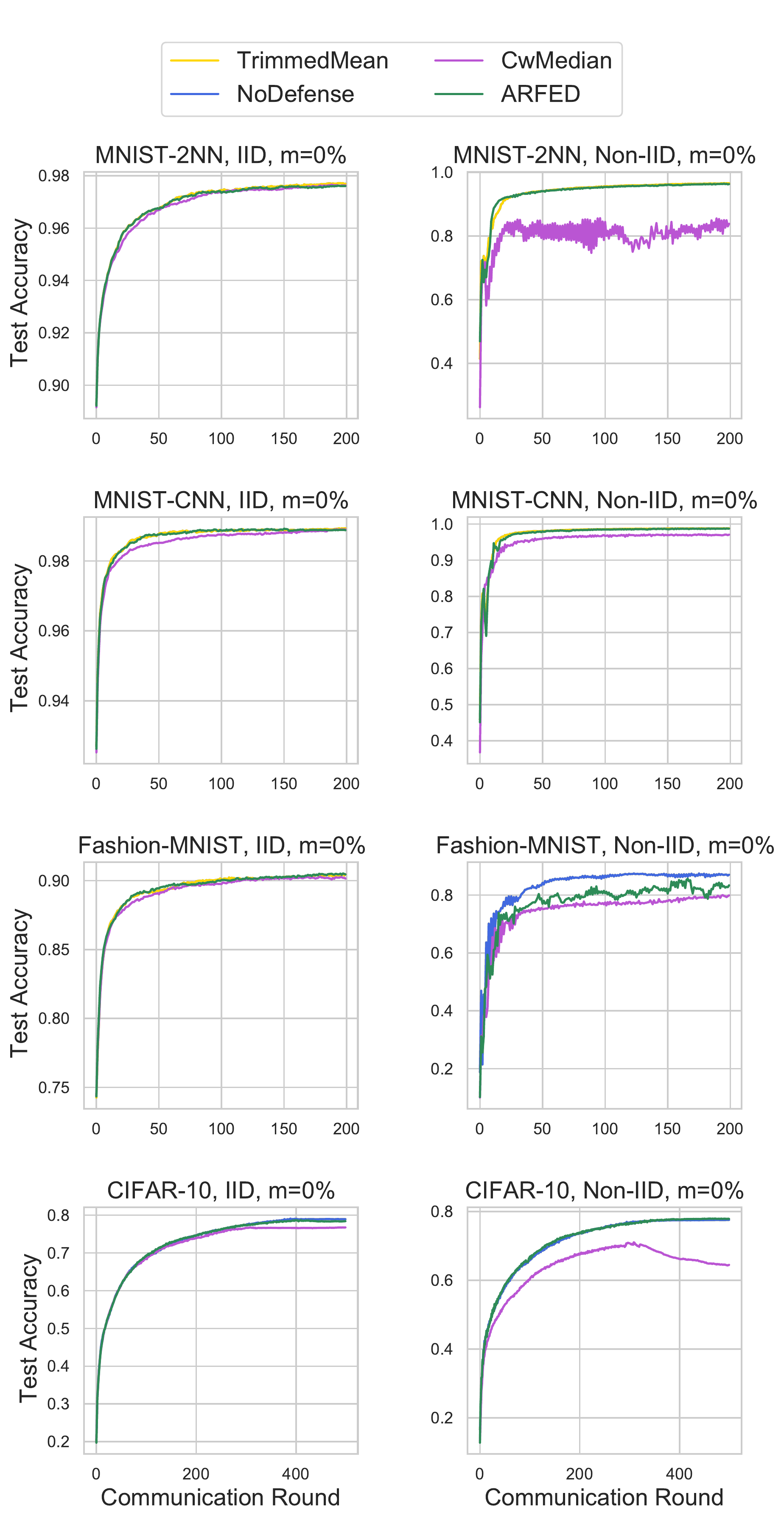}
  \caption{Accuracy curves of different strategies when all participants are trusted.}
  \label{fig:no_attack_fig}
\end{figure}

A robust defense strategy should not cause any noticeable performance loss when there is no attack in the FL system. Figure~\ref{fig:no_attack_fig} and Table~\ref{tab:all_trusted_table} show the results of experiments when all participants are trusted (when there is no attack on any participant). Incorporating TrimmedMean and CwMedian strategies into FedAvg does not cause any performance degradation in the IID setting. Although the performance of the CwMedian strategy is slightly worse in Fashion-MNIST and CIFAR10, it can be tolerable in an FL setting. On the other hand, when local datasets are Non-IID, the CwMedian strategy causes significant performance degradation, which points to the questionability of the method. When the malicious participant ratio is zero, in other words, when there is no attack in the system, no participants are discarded from the aggregation step; therefore, TrimmedMean gives the same result as NoDefense.

\begin{table}[htb!]
\caption{Accuracies on test sets when all participants are trusted (i.e., $m=0\%$). The worst results are bold.}
\label{tab:all_trusted_table}
\begin{adjustbox}{width=\columnwidth,center}
\begin{tabular}{ccccccccc}
\textbf{~}       & \multicolumn{2}{c}{\textbf{MNIST-2NN}} & \multicolumn{2}{c}{\textbf{MNIST-CNN}} & \multicolumn{2}{c}{\textbf{Fashion-MNIST}} & \multicolumn{2}{c}{\textbf{CIFAR10}}  \\ 
\hline\hline
\textbf{IID}     & \textbf{min}  & \textbf{max}           & \textbf{min}  & \textbf{max}           & \textbf{min}  & \textbf{max}               & \textbf{min}  & \textbf{max}          \\ 
\hline
\textbf{NoDefense}       & 97.7          & 97.7                   & 98.9          & 98.9                   & 90.4          & 90.5                       & 78.9          & 79.0                  \\
\textbf{TrimmedMean$^1$}    & 97.7          & 97.7                   & 98.9          & 98.9                   & 90.4          & 90.5                       & 78.9          & 79.0                  \\
\textbf{CwMedian}           & \textbf{97.6} & \textbf{97.6}          & 98.9          & 98.9                   & \textbf{90.1} & \textbf{90.2}              & \textbf{76.7} & \textbf{76.8}         \\
\textbf{ARFED}           & \textbf{97.6} & \textbf{97.6}          & 98.9          & 98.9                   & 90.4          & 90.5                       & 78.4          & 78.4                  \\
                 &               &                        &               &                        &               &                            &               &                       \\ 
\hline\hline
\textbf{Non-IID} & \textbf{min}  & \textbf{max}           & \textbf{min}  & \textbf{max}           & \textbf{min}  & \textbf{max}               & \textbf{min}  & \textbf{max}          \\ 
\hline
\textbf{NoDefense}       & 96.4          & 96.6                   & 98.8          & 98.8                   & 86.8          & 87.4                       & 77.5          & 77.6                  \\
\textbf{TrimmedMean$^1$}    & 96.4          & 96.6                   & 98.8          & 98.8                   & 86.8          & 87.4                       & 77.5          & 77.6                  \\
\textbf{CwMedian}           & \textbf{80.7} & \textbf{85.3}          & \textbf{96.9} & \textbf{97.1}          & \textbf{79.1} & \textbf{80.0}              & \textbf{64.3} & \textbf{64.6}         \\
\textbf{ARFED}           & 96.2          & 96.4                   & 98.7          & 98.8                   & 82.4          & 84.3                       & 77.8          & 77.9     \\
\hline
\end{tabular}
\end{adjustbox}
\vspace{2mm}
\footnotesize{$^1$ The same results as NoDefense}\\
\end{table}

\subsection{Label Flipping Attacks}\label{label_flipping_attacks}

Table~\ref{tab:label_flipping_iid_table} and Table~\ref{tab:label_flipping_noniid_table} show that as the ratio of malicious participants increases, vanilla FedAvg cannot avoid performance degradation, which requires that a defense mechanism should be incorporated. As indicated by Figure~\ref{fig:mnist_2nn_lf_fig} when malicious participants are organized, degradation becomes more severe and oscillation increases. The most performance degradation occurs when attacks are organized in Non-IID setting. 

For IID cases of MNIST-2NN, MNIST-CNN, and Fashion-MNIST, ARFED achieves a slightly higher accuracy score most of the time, but differences between ARFED, CwMedian, and trimmed mean are not significant. When comparing with all-trusted performance, all strategies can tolerate the negative effects of the label flipping attack. For IID cases of CIFAR10 experiments, ARFED achieves noticeably better performance than the others.

When the data of participants is Non-IID, CwMedian strategy performed worse than even the vanilla FedAvg. Although trimmed mean achieves better performance than CwMedian and can remove the performance loss, it can be said that ARFED outperforms both of them and gives the highest accuracy scores among all these methods. In other words, ARFED successfully defends against malicious participants and gets an accuracy score very close to when all participants are trusted. The performance of accuracy curves for MNIST-2NN can be examined in Figure \ref{fig:mnist_2nn_lf_fig} and for other datasets in Appendix~\ref{appendix:experiments}.

\begin{table}[hbtp!]
\caption{Accuracies under label flipping attacks at different attacker ratios in IID settings. The best results are bold.}
\label{tab:label_flipping_iid_table}
\begin{adjustbox}{width=\columnwidth, center}
\begin{tabular}{cccccccccc}
~                & ~                     & \multicolumn{4}{c}{\textbf{Organized~}}                       & \multicolumn{4}{c}{\textbf{Independent}}                       \\ 
\hline\hline
                 &                       & \multicolumn{2}{c}{m=10\%}    & \multicolumn{2}{c}{m=20\%}    & \multicolumn{2}{c}{m=10\%}    & \multicolumn{2}{c}{m=20\%}     \\
\textbf{~}       & \textbf{~}            & min           & max           & min           & max           & min           & max           & min           & max            \\ 
\hline\hline
                 & \textbf{NoDefense}   & 92.9          & 97.6          & 72.4          & 97.3          & 97.0          & 97.2          & 91.3          & 97.1           \\
\textbf{MNIST}   & \textbf{CwMedian}       & 97.4          & 97.4          & 97.2          & 97.3          & 97.4          & 97.5          & 97.0          & 97.1           \\
\textbf{2NN}     & \textbf{TrimmedMean} & 97.5          & 97.5          & 96.9          & 97.0          & 97.5          & 97.5          & 97.1          & 97.1           \\
                 & \textbf{ARFED}       & \textbf{97.6} & \textbf{97.7} & \textbf{97.4} & \textbf{97.5} & \textbf{97.6} & \textbf{97.6} & \textbf{97.4} & \textbf{97.5}  \\ 
\hline\hline
                 &                       &               &               &               &               &               &               &               &                \\
                 & \textbf{NoDefense}   & 94.2          & \textbf{99.0} & 75.4          & \textbf{99.0} & 97.8          & 99.0          & 96.2          & 98.9           \\
\textbf{MNIST}   & \textbf{CwMedian}       & \textbf{98.9} & 98.9          & 98.8          & 98.8          & 98.9          & 98.9          & \textbf{98.9} & 98.9           \\
\textbf{CNN}     & \textbf{TrimmedMean} & \textbf{98.9} & 98.9          & 98.8          & 98.9          & \textbf{99.0} & \textbf{99.0} & \textbf{98.9} & \textbf{99.0}  \\
\textbf{~}       & \textbf{ARFED}       & \textbf{98.9} & 98.9          & \textbf{98.9} & 98.9          & 98.9          & 98.9          & \textbf{98.9} & 98.9           \\ 
\hline\hline
\textbf{~}       &                       &               &               &               &               &               &               &               &                \\
                 & \textbf{NoDefense}   & 87.8          & 89.3          & 68.6          & 88.9          & 89.2          & 89.5          & 83.7          & 89.0           \\
\textbf{Fashion} & \textbf{CwMedian}       & 89.8          & 89.9          & 88.6          & 88.7          & 89.6          & 89.7          & 89.2          & 89.3           \\
\textbf{MNIST}   & \textbf{TrimmedMean} & 90.0          & 90.1          & 88.8          & 88.9          & 89.9          & 90.0          & 89.0          & 89.2           \\
~                & \textbf{ARFED}       & \textbf{90.5} & \textbf{90.7} & \textbf{90.3} & \textbf{90.4} & \textbf{90.2} & \textbf{90.3} & \textbf{90.2} & \textbf{90.3}  \\ 
\hline\hline
\textbf{~}       &                       &               &               &               &               &               &               &               &                \\
                 & \textbf{NoDefense}   & 72.7          & 72.7          & 65.8          & 65.9          & 72.8          & 72.8          & 69.7          & 69.8           \\
\textbf{CIFAR10} & \textbf{CwMedian}       & 75.6          & 75.7          & 73.3          & 73.4          & 73.8          & 73.8          & 73.5          & 73.6           \\
                 & \textbf{TrimmedMean} & 75.6          & 75.7          & 73.3          & 73.4          & 73.8          & 73.8          & 73.5          & 73.6           \\
~                & \textbf{ARFED}       & \textbf{76.2} & \textbf{76.2} & \textbf{77.0} & \textbf{77.2} & \textbf{77.7} & \textbf{77.8} & \textbf{75.6} & \textbf{75.7}  \\
\hline\hline
\end{tabular}
\end{adjustbox}
\end{table}


\begin{table}[h!]
\caption{Accuracies under label flipping attacks at different attacker ratios in Non-IID settings. The best results are bold.}
\label{tab:label_flipping_noniid_table}
\begin{adjustbox}{width=\columnwidth, center}
\begin{tabular}{cccccccccc}
~                & ~                     & \multicolumn{4}{c}{\textbf{Organized~}}                       & \multicolumn{4}{c}{\textbf{Independent}}                       \\ 
\hline\hline
                 &                       & \multicolumn{2}{c}{m=10\%}    & \multicolumn{2}{c}{m=20\%}    & \multicolumn{2}{c}{m=10\%}    & \multicolumn{2}{c}{m=20\%}     \\
\textbf{~}       & \textbf{~}            & min           & max           & min           & max           & min           & max           & min           & max            \\ 
\hline\hline
                 & \textbf{NoDefense}   & 92.4          & 95.8          & 83.8          & 88.3          & 93.7          & 95.3          & 89.3          & 94.0           \\
\textbf{MNIST}   & \textbf{CwMedian}       & 75.1          & 83.8          & 67.7          & 75.0          & 80.8          & 83.3          & 67.8          & 75.9           \\
\textbf{2NN}     & \textbf{TrimmedMean} & 94.6          & 95.5          & 79.9          & 87.7          & 95.1          & 95.4          & 80.9          & 89.7           \\
                 & \textbf{ARFED}       & \textbf{96.1} & \textbf{96.3} & \textbf{95.6} & \textbf{96.1} & \textbf{96.0} & \textbf{96.3} & \textbf{95.5} & \textbf{96.1}  \\ 
\hline\hline
                 &                       &               &               &               &               &               &               &               &                \\
                 & \textbf{NoDefense}   & 95.5          & 98.0          & 83.4          & 91.6          & 97.0          & 98.3          & 95.3          & 96.6           \\
\textbf{MNIST}   & \textbf{CwMedian}       & 96.4          & 96.7          & 91.4          & 93.1          & 95.4          & 95.9          & 93.5          & 94.2           \\
\textbf{CNN}     & \textbf{TrimmedMean} & 98.5          & 98.6          & 97.3          & 97.6          & 98.6          & 98.6          & 97.2          & 97.6           \\
~                & \textbf{ARFED}       & \textbf{98.6} & \textbf{98.7} & \textbf{98.5} & \textbf{98.6} & \textbf{98.7} & \textbf{98.8} & \textbf{98.6} & \textbf{98.7}  \\ 
\hline\hline
\textbf{~}       &                       &               &               &               &               &               &               &               &                \\
                 & \textbf{NoDefense}   & 83.0          & 85.7          & 73.7          & 79.2          & 84.3          & 86.7          & 81.5          & 84.6           \\
\textbf{Fashion} & \textbf{CwMedian}       & 79.9          & 80.6          & 76.0          & 76.7          & 78.7          & 79.6          & 77.4          & 78.5           \\
\textbf{MNIST}   & \textbf{TrimmedMean} & 86.7          & 87.5          & 82.7          & 83.4          & 85.8          & 86.7          & 83.8          & 84.3           \\
~                & \textbf{ARFED}       & \textbf{87.7} & \textbf{88.8} & \textbf{84.2} & \textbf{87.6} & \textbf{87.5} & \textbf{88.5} & \textbf{85.1} & \textbf{87.6}  \\ 
\hline\hline
\textbf{~}       &                       &               &               &               &               &               &               &               &                \\
                 & \textbf{NoDefense}   & 75.2          & 75.3          & 70.8          & 70.9          & 74.4          & 74.4          & 70.2          & 70.3           \\
\textbf{CIFAR10} & \textbf{CwMedian}       & 55.0          & 55.9          & 54.5          & 54.9          & 56.7          & 57.2          & 52.0          & 52.7           \\
                 & \textbf{TrimmedMean} & 76.0          & 76.0          & 72.6          & 72.7          & 75.4          & 75.4          & 74.0          & 74.1           \\
~                & \textbf{ARFED}       & \textbf{78.0} & \textbf{78.1} & \textbf{76.8} & \textbf{76.9} & \textbf{77.3} & \textbf{77.4} & \textbf{76.2} & \textbf{76.3}  \\
\hline\hline
\end{tabular}
\end{adjustbox}
\end{table}


\begin{figure}[h!]
  \centering
  \includegraphics[width=\columnwidth]{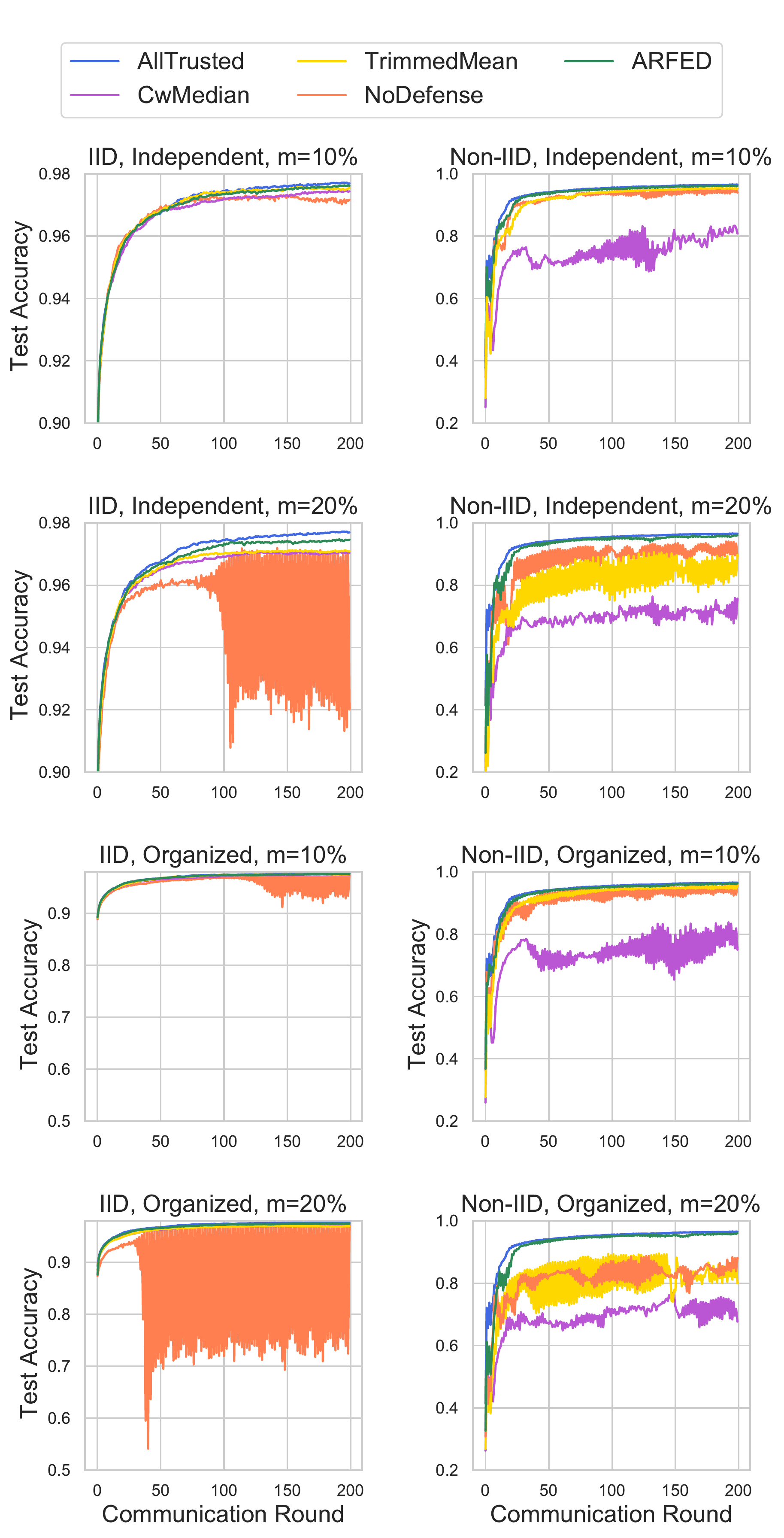}
  \caption{Accuracy curves of different strategies for MNIST-2NN under label flipping attacks at different attacker ratios.}
  \label{fig:mnist_2nn_lf_fig}
\end{figure}


\subsection{Byzantine Attacks}\label{byz_attacks}

In the literature, it has been shown that Byzantine attacks are more effective than data poisoning attacks \cite{fang2020local, sattler2020byzantine, bagdasaryan2020backdoor}. Our experiments, which have shown that there is dramatic performance degradation under Byzantine attacks, are in line with the previous studies. In addition, the performance degradation caused by organized attackers is more severe than independent attackers.

When the data distributions of the participants are Non-IID, the performance scores get worse compared to IID cases and Non-IID attacks with organized attackers are the most harmful case for the performance. Again, as the number of malicious participants increases, the performance degradation increases, too.

For IID cases, all defense strategies are able to tolerate the negative effects of Byzantine attacks as if there has been no malicious participant (Table~\ref{tab:byz_iid_table} and Figure~\ref{fig:mnist_2nn_byz_fig}). For Non-IID cases, the CwMedian is able to prevent performance degradation up to a degree; however, it can not eliminate the performance degradation caused by attacks as well as ARFED and trimmed mean. ARFED gets better scores than CwMedian for all data sets. As Figure~\ref{fig:mnist_2nn_byz_fig} and Table~\ref{tab:byz_noniid_table} shows CwMedian can catch the ARFED for only Non-IID experiments of MNIST-CNN. 


\begin{table}[hb!]
\caption{Accuracies under Byzantine attacks at different attacker ratios in IID settings. The best results are bold.}
\label{tab:byz_iid_table}
\begin{adjustbox}{width=\columnwidth, center}
\begin{tabular}{cccccccccc}
~                & ~                     & \multicolumn{4}{c}{\textbf{Organized~}}                       & \multicolumn{4}{c}{\textbf{Independent}}                       \\ 
\hline\hline
                 &                       & \multicolumn{2}{c}{m=10\%}    & \multicolumn{2}{c}{m=20\%}    & \multicolumn{2}{c}{m=10\%}    & \multicolumn{2}{c}{m=20\%}     \\
\textbf{~}       & \textbf{~}            & min           & max           & min           & max           & min           & max           & min           & max            \\ 
\hline\hline
                 & \textbf{NoDefense}   & 57.0          & 68.9          & 31.9          & 43.9          & 86.0          & 90.7          & 76.2          & 86.1           \\
\textbf{MNIST}   & \textbf{CwMedian}       & 97.2          & 97.2          & 97.3          & 97.4          & 97.2          & 97.3          & 97.3          & 97.3           \\
\textbf{2NN}     & \textbf{TrimmedMean} & 97.4          & 97.5          & 97.4          & 97.5          & \textbf{97.5} & \textbf{97.5} & 97.4          & 97.4           \\
                 & \textbf{ARFED}       & \textbf{97.5} & \textbf{97.5} & \textbf{97.5} & \textbf{97.6} & \textbf{97.5} & \textbf{97.5} & \textbf{97.5} & \textbf{97.6}  \\ 
\hline\hline
                 &                       &               &               &               &               &               &               &               &                \\
                 & \textbf{NoDefense}   & 54.1          & 79.3          & 10.0          & 17.0          & 91.9          & 95.0          & 70.8          & 87.2           \\
\textbf{MNIST}   & \textbf{CwMedian}       & 98.9          & 98.9          & \textbf{98.9} & 98.9          & \textbf{98.9} & 98.9          & 98.8          & 98.8           \\
\textbf{CNN}     & \textbf{TrimmedMean} & \textbf{99.0} & \textbf{99.0} & \textbf{98.9} & \textbf{99.0} & \textbf{98.9} & \textbf{99.0} & \textbf{98.9} & \textbf{98.9}  \\
\textbf{~}       & \textbf{ARFED}       & 98.9          & 98.9          & 98.8          & 98.8          & \textbf{98.9} & \textbf{99.0} & 98.8          & 98.8           \\ 
\hline\hline
\textbf{~}       &                       &               &               &               &               &               &               &               &                \\
                 & \textbf{NoDefense}   & 15.3          & 48.2          & 9.80          & 21.1          & 65.0          & 79.2          & 36.8          & 60.2           \\
\textbf{Fashion} & \textbf{CwMedian}       & 89.9          & 90.1          & 90.1          & 90.2          & 90.0          & 90.2          & 90.0          & 90.1           \\
\textbf{MNIST}   & \textbf{TrimmedMean} & \textbf{90.2} & \textbf{90.3} & 90.2          & 90.3          & \textbf{90.4} & \textbf{90.4} & 90.2          & 90.3           \\
\textbf{~}       & \textbf{ARFED}       & \textbf{90.2} & \textbf{90.3} & \textbf{90.4} & \textbf{90.5} & 90.2          & 90.3          & \textbf{90.3} & \textbf{90.4}  \\ 
\hline\hline
\textbf{~}       &                       &               &               &               &               &               &               &               &                \\
                 & \textbf{NoDefense}   & 8.7           & 12.2          & 8.3           & 11.3          & 8.9           & 13.2          & 8.9           & 11.1           \\
\textbf{CIFAR10} & \textbf{CwMedian}       & 75.7          & 75.8          & 75.0          & 75.0          & \textbf{77.6} & \textbf{77.7} & 75.1          & 75.1           \\
                 & \textbf{TrimmedMean} & \textbf{78.8} & \textbf{78.9} & 75.6          & 75.7          & 76.8          & 76.9          & 77.0          & 77.1           \\
\textbf{~}       & \textbf{ARFED}       & 77.8          & 77.9          & \textbf{77.4} & \textbf{77.5} & 77.1          & 77.2          & \textbf{77.3} & \textbf{77.3}  \\
\hline\hline
\end{tabular}
\end{adjustbox}
\end{table}


\begin{table}[h!]
\caption{Accuracies under Byzantine attacks at different attacker ratios in Non-IID settings. The best results are bold.}
\label{tab:byz_noniid_table}
\begin{adjustbox}{width=\columnwidth, center}
\begin{tabular}{cccccccccc}
~                & ~                     & \multicolumn{4}{c}{\textbf{Organized~}}                       & \multicolumn{4}{c}{\textbf{Independent}}                       \\ 
\hline\hline
                 &                       & \multicolumn{2}{c}{m=10\%}    & \multicolumn{2}{c}{m=20\%}    & \multicolumn{2}{c}{m=10\%}    & \multicolumn{2}{c}{m=20\%}     \\
\textbf{~}       & \textbf{~}            & min           & max           & min           & max           & min           & max           & min           & max            \\ 
\hline\hline
                 & \textbf{NoDefense}   & 26.1          & 36.0          & 14.9          & 26.6          & 45.6          & 61.6          & 16.5          & 34.1           \\
\textbf{MNIST}   & \textbf{CwMedian}       & 85.6          & 89.5          & 90.3          & 92.7          & 83.4          & 88.8          & 89.3          & 90.9           \\
\textbf{2NN}     & \textbf{TrimmedMean} & 95.8          & 96.1          & 95.4          & 95.6          & 95.9          & 96.1          & 94.8          & 95.0           \\
                 & \textbf{ARFED}       & \textbf{96.1} & \textbf{96.2} & \textbf{95.9} & \textbf{96.1} & \textbf{96.1} & \textbf{96.2} & \textbf{95.9} & \textbf{96.1}  \\ 
\hline\hline
                 &                       &               &               &               &               &               &               &               &                \\
                 & \textbf{NoDefense}   & 15.6          & 33.5          & 9.00          & 15.7          & 46.1          & 72.4          & 17.4          & 40.2           \\
\textbf{MNIST}   & \textbf{CwMedian}       & 97.4          & 97.5          & 97.6          & 97.7          & 97.0          & 97.2          & 96.8          & 97.1           \\
\textbf{CNN}     & \textbf{TrimmedMean} & \textbf{98.8} & \textbf{98.8} & \textbf{98.6} & \textbf{98.7} & \textbf{98.8} & \textbf{98.8} & 98.5          & 98.6           \\
\textbf{~}       & \textbf{ARFED}       & 98.7          & \textbf{98.8} & \textbf{98.6} & \textbf{98.7} & 98.7          & \textbf{98.8} & \textbf{98.6} & \textbf{98.7}  \\ 
\hline\hline
\textbf{~}       &                       &               &               &               &               &               &               &               &                \\
                 & \textbf{NoDefense}   & 5.70          & 23.6          & 8.60          & 18.0          & 19.7          & 39.6          & 10.0          & 22.2           \\
\textbf{Fashion} & \textbf{CwMedian}       & 79.4          & 80.8          & 78.4          & 79.6          & 79.7          & 80.8          & 77.7          & 78.3           \\
\textbf{MNIST}   & \textbf{TrimmedMean} & \textbf{84.9} & \textbf{86.2} & \textbf{85.9} & \textbf{86.6} & \textbf{84.2} & \textbf{86.0} & 83.2          & 84.3           \\
\textbf{~}       & \textbf{ARFED}       & 82.3          & 85.6          & 85.3          & 86.5          & 80.8          & 83.8          & \textbf{84.5} & \textbf{86.3}  \\ 
\hline\hline
\textbf{~}       &                       &               &               &               &               &               &               &               &                \\
                 & \textbf{NoDefense}   & 8.9           & 11.4          & 8.2           & 11.5          & 7.9           & 10.3          & 8.7           & 13.4           \\
\textbf{CIFAR10} & \textbf{CwMedian}       & 57.7          & 59.0          & 71.8          & 72.0          & 62.0          & 62.3          & 58.0          & 61.2           \\
                 & \textbf{TrimmedMean} & 76.4          & 76.5          & 76.3          & 76.4          & \textbf{76.6} & \textbf{76.6} & 75.5          & 75.6           \\
\textbf{~}       & \textbf{ARFED}       & \textbf{77.6} & \textbf{77.6} & \textbf{77.5} & \textbf{77.6} & 75.2          & 75.3          & \textbf{77.2} & \textbf{77.2}  \\
\hline\hline
\end{tabular}
\end{adjustbox}
\end{table}

\begin{figure}[h!]
  \centering
  \includegraphics[width=\columnwidth]{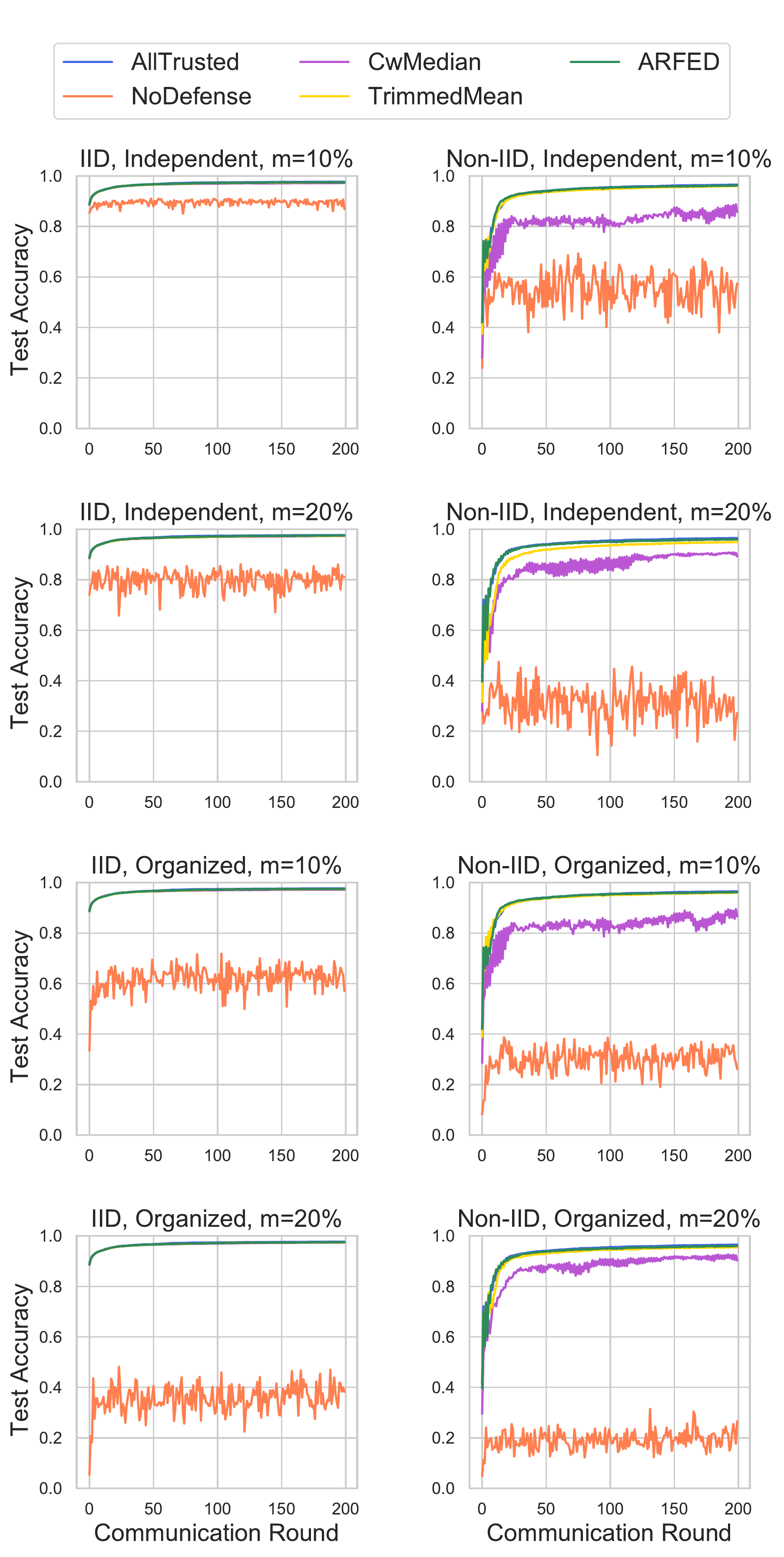}
  \caption{Accuracy curves of different strategies for MNIST-2NN under Byzantine attacks at different attacker ratios.}
  \label{fig:mnist_2nn_byz_fig}
\end{figure}

\subsection{Partial Knowledge Attack}\label{partial_knowledge_attacks}

In line with previous experiments on other attack types, Table~\ref{tab:fang_iid_table} and Table~\ref{tab:fang_noniid_table} show that Non-IID attacks are more severe than IID attacks. Moreover, as the ratio of malicious participants increases, the performance degrades more and when malicious participants are organized, the degradation becomes more severe for both IID and Non-IID cases.

For IID cases, all strategies can reverse the performance degradation caused by the partial knowledge attack but ARFED can achieve the highest score in all cases. The difference between ARFED and other defense strategies becomes more visible for Fashion-MNIST and CIFAR10.

When the data of participants are Non-IID and the attackers are independent, CwMedian achieves worse accuracy scores than the vanilla FedAvg while trimmed mean can tolerate the performance loss caused by the partial knowledge attack. Still, ARFED gets the highest scores among all of them. Moreover, When the data of participants are Non-IID and the attackers are organized, CwMedian can provide a slight improvement and TrimmedMean can get rid of the the performance loss up to a point. ARFED successfully defends against the attack and achieves accuracy scores very close to when all participants are trusted. The performance of accuracy curves for MNIST-2NN can be examined in Figure \ref{fig:mnist_2nn_fang_fig}.

\begin{table}[htbp!]
\caption{Accuracies under partial knowledge attacks at different attacker ratios in IID settings. The best results are bold.}
\label{tab:fang_iid_table}
\begin{adjustbox}{width=\columnwidth, center}
\begin{tabular}{cccccccccc}
~                & ~                     & \multicolumn{4}{c}{\textbf{Organized~}}                       & \multicolumn{4}{c}{\textbf{Independent}}                       \\ 
\hline\hline
                 &                       & \multicolumn{2}{c}{m=10\%}    & \multicolumn{2}{c}{m=20\%}    & \multicolumn{2}{c}{m=10\%}    & \multicolumn{2}{c}{m=20\%}     \\
\textbf{~}       & \textbf{~}            & min           & max           & min           & max           & min           & max           & min           & max            \\ 
\hline\hline
                 & \textbf{NoDefense}   & 59.0          & 87.1          & 11.2          & 12.6          & 95.2          & 95.4          & 94.6          & 94.8           \\
\textbf{MNIST}   & \textbf{CwMedian}       & 96.1          & 96.2          & 94.1          & 94.4          & 96.9          & 97.0          & 96.6          & 96.6           \\
\textbf{2NN}     & \textbf{TrimmedMean} & 95.1          & 95.2          & 88.3          & 94.7          & 96.7          & 96.8          & 95.9          & 96.0           \\
                 & \textbf{ARFED}       & \textbf{97.4} & \textbf{97.5} & \textbf{97.5} & \textbf{97.5} & \textbf{97.4} & \textbf{97.5} & \textbf{97.5} & \textbf{97.5}  \\ 
\hline\hline
                 &                       &               &               &               &               &               &               &               &                \\
                 & \textbf{NoDefense}   & 98.1          & 98.2          & 86.9          & 90.7          & 97.9          & 98.0          & 97.3          & 97.4           \\
\textbf{MNIST}   & \textbf{CwMedian}       & 98.4          & 98.4          & 96.7          & 96.8          & 98.8          & 98.9          & 98.7          & 98.7           \\
\textbf{CNN}     & \textbf{TrimmedMean} & 97.6          & 97.7          & 96.5          & 97.7          & 98.5          & 98.5          & 98.1          & 98.1           \\
~                & \textbf{ARFED}       & \textbf{98.9} & \textbf{99.0} & \textbf{98.9} & \textbf{98.9} & \textbf{98.9} & \textbf{99.0} & \textbf{98.9} & \textbf{98.9}  \\ 
\hline\hline
\textbf{~}       &                       &               &               &               &               &               &               &               &                \\
                 & \textbf{NoDefense}   & 84.4          & 84.8          & 73.1          & 74.7          & 88.8          & 89.1          & 84.7          & 85.0           \\
\textbf{Fashion} & \textbf{CwMedian}       & 88.9          & 89.0          & 85.9          & 86.0          & 89.7          & 89.8          & 89.2          & 89.3           \\
\textbf{MNIST}   & \textbf{TrimmedMean} & 88.0          & 88.2          & 84.4          & 84.7          & 89.7          & 89.9          & 88.8          & 89.0           \\
~                & \textbf{ARFED}       & \textbf{90.4} & \textbf{90.5} & \textbf{90.1} & \textbf{90.2} & \textbf{90.2} & \textbf{90.3} & \textbf{90.3} & \textbf{90.4}  \\ 
\hline\hline
\textbf{~}       &                       &               &               &               &               &               &               &               &                \\
                 & \textbf{NoDefense}   & 68.3          & 68.4          & 50.4          & 50.5          & 74.1          & 74.1          & 66.2          & 66.3           \\
\textbf{CIFAR10} & \textbf{CwMedian}       & 72.6          & 72.6          & 65.7          & 65.8          & 73.8          & 73.9          & 71.7          & 71.8           \\
                 & \textbf{TrimmedMean} & 71.2          & 71.3          & 62.4          & 62.5          & 74.0          & 74.1          & 70.3          & 70.4           \\
~                & \textbf{ARFED}       & \textbf{78.2} & \textbf{78.3} & \textbf{76.9} & \textbf{77.0} & \textbf{77.4} & \textbf{77.4} & \textbf{78.5} & \textbf{78.5} \\
\hline\hline
\end{tabular}
\end{adjustbox}
\end{table}

\begin{table}[htp!]
\caption{Accuracies under partial knowledge attacks at different attacker ratios in Non-IID settings. The best results are bold.}
\label{tab:fang_noniid_table}
\begin{adjustbox}{width=\columnwidth, center}
\begin{tabular}{cccccccccc}
~                & ~                     & \multicolumn{4}{c}{\textbf{Organized~}}                       & \multicolumn{4}{c}{\textbf{Independent}}                       \\ 
\hline\hline
                 &                       & \multicolumn{2}{c}{m=10\%}    & \multicolumn{2}{c}{m=20\%}    & \multicolumn{2}{c}{m=10\%}    & \multicolumn{2}{c}{m=20\%}     \\
\textbf{~}       & \textbf{~}            & min           & max           & min           & max           & min           & max           & min           & max            \\ 
\hline\hline
                 & \textbf{NoDefense}   & 50.9          & 60.7          & 13.2          & 16.9          & 87.5          & 92.3          & 85.2          & 90.7           \\
\textbf{MNIST}   & \textbf{CwMedian}       & 62.6          & 67.2          & 21.1          & 35.2          & 78.4          & 84.6          & 63.3          & 78.8           \\
\textbf{2NN}     & \textbf{TrimmedMean} & 77.1          & 81.5          & 32.0          & 38.0          & 94.1          & 94.7          & 91.0          & 91.8           \\
                 & \textbf{ARFED}       & \textbf{96.1} & \textbf{96.3} & \textbf{95.9} & \textbf{96.2} & \textbf{96.1} & \textbf{96.3} & \textbf{95.9} & \textbf{96.2}  \\ 
\hline\hline
                 &                       &               &               &               &               &               &               &               &                \\
                 & \textbf{NoDefense}   & 51.8          & 58.8          & 16.5          & 24.7          & 97.6          & 97.9          & 91.2          & 97.1           \\
\textbf{MNIST}   & \textbf{CwMedian}       & 80.4          & 87.9          & 38.1          & 39.1          & 96.5          & 96.8          & 95.6          & 96.0           \\
\textbf{CNN}     & \textbf{TrimmedMean} & 88.0          & 88.7          & 27.0          & 43.6          & 98.3          & 98.4          & 97.3          & 97.6           \\
~                & \textbf{ARFED}       & \textbf{98.8} & \textbf{98.9} & \textbf{98.7} & \textbf{98.8} & \textbf{98.8} & \textbf{99.0} & \textbf{98.6} & \textbf{98.7}  \\ 
\hline\hline
\textbf{~}       &                       &               &               &               &               &               &               &               &                \\
                 & \textbf{NoDefense}   & 42.0          & 51.4          & 9.3           & 21.9          & 79.3          & 83.5          & 75.5          & 80.6           \\
\textbf{Fashion} & \textbf{CwMedian}       & 53.6          & 54.6          & 26.8          & 28.0          & 74.8          & 75.8          & 66.4          & 68.0           \\
\textbf{MNIST}   & \textbf{TrimmedMean} & 54.4          & 58.0          & 38.6          & 51.3          & 81.6          & 83.5          & 79.8          & 81.2           \\
~                & \textbf{ARFED}       & \textbf{82.0} & \textbf{82.8} & \textbf{85.6} & \textbf{87.5} & \textbf{79.6} & \textbf{84.5} & \textbf{86.0} & \textbf{87.3}  \\ 
\hline\hline
\textbf{~}       &                       &               &               &               &               &               &               &               &                \\
                 & \textbf{NoDefense}   & 52.7          & 52.9          & 30.9          & 31.7          & 72.0          & 72.0          & 65.5          & 65.6           \\
\textbf{CIFAR10} & \textbf{CwMedian}       & 53.0          & 53.3          & 38.6          & 38.8          & 61.7          & 62.1          & 45.8          & 46.3           \\
                 & \textbf{TrimmedMean} & 68.4          & 68.5          & 46.7          & 46.8          & 72.5          & 72.6          & 66.0          & 66.0           \\
~                & \textbf{ARFED}       & \textbf{77.8} & \textbf{77.8} & \textbf{77.3} & \textbf{77.3} & \textbf{77.8} & \textbf{77.9} & \textbf{76.4} & \textbf{76.5}  \\
\hline\hline
\end{tabular}
\end{adjustbox}
\end{table}

\begin{figure}[h!]
  \centering
  \includegraphics[width=\columnwidth]{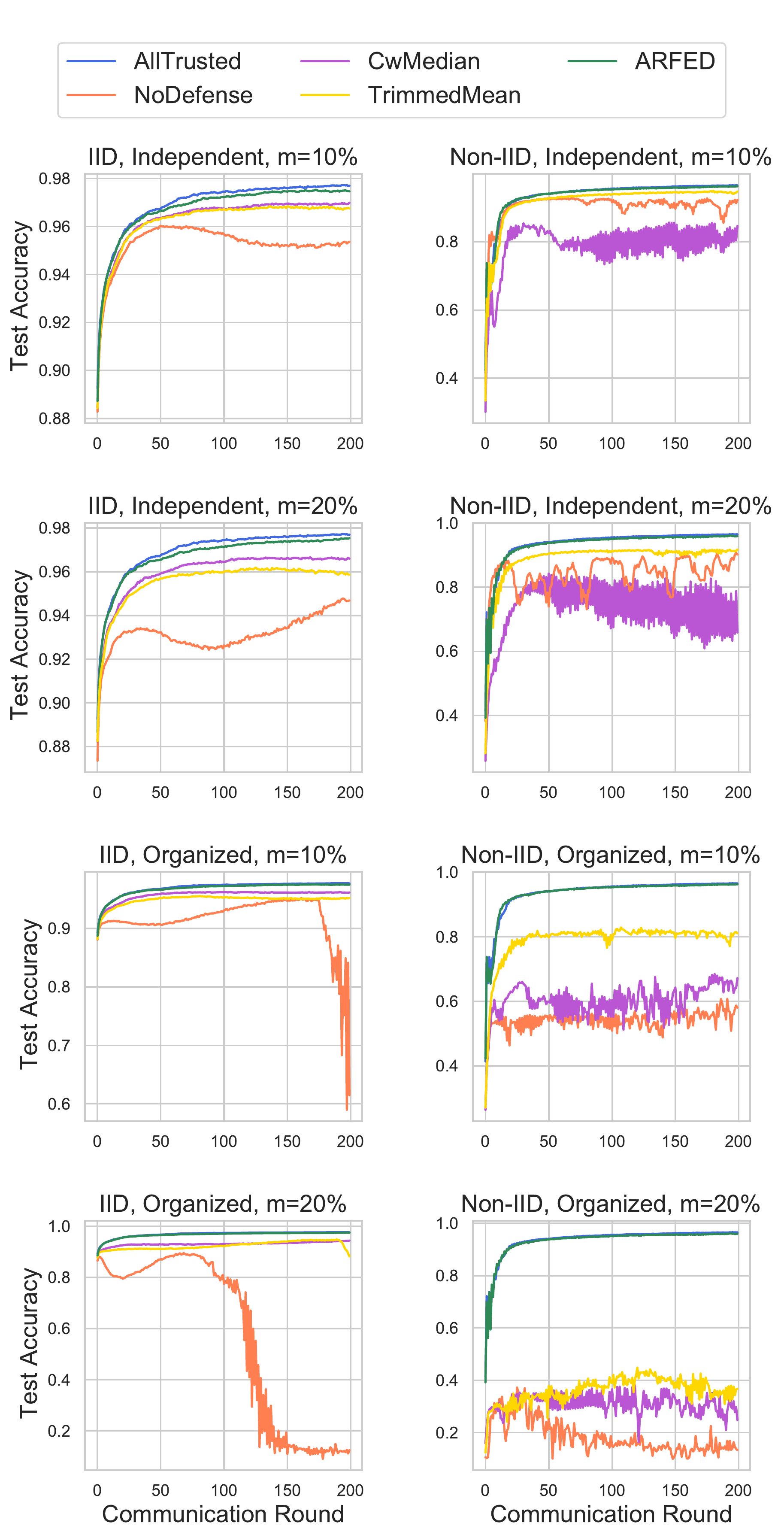}
  \caption{Accuracy curves of different strategies for MNIST-2NN under partial knowledge attacks at different attacker ratios.}
  \label{fig:mnist_2nn_fang_fig}
\end{figure}


\section{Conclusion}\label{conclusion}

This study proposes ARFED, an assumption-free attack-resistant federated averaging algorithm based on outlier elimination, and conducts comprehensive experiments in various FL settings. These experiments reveal that Byzantine attacks and partial knowledge attacks are dramatically more severe than label flipping attacks. Moreover, attacks in the Non-IID cases are more effective than IID cases and organized attackers can severely compromise the performance of the main model more compared to independent attackers. 

Although CwMedian, TrimmedMean, and ARFED tolerate performance loss in the presence of attacks in IID cases, the CwMedian performs poorly in Non-IID cases; it may even perform worse than the vanilla FedAvg. For Non-IID cases, ARFED shows better performance recovery than CwMedian in all attack types. Moreover, ARFED outperforms TrimmedMean in label flipping attacks and partial knowledge attacks, but they get similar results in Byzantine attacks. The likely reason for this is that the parameter updates sent by malicious participants are extreme and lie in the distribution's tails for Byzantine attacks. In this way, TrimmedMean can detect poisoned parameter updates more easily. On the other hand, in label flipping attacks and partial knowledge attacks, changes in parameter updates are more likely to be moderate; therefore, the TrimmedMean cannot provide the same performance. Moreover, it is also worth keeping in mind that TrimmedMean needs information about the malicious participant ratio in the system, while ARFED does not make such an assumption.

We put forward experimental evidence to show that ARFED removes performance loss even under organized attacks and in Non-IID cases. There are many attack-robust aggregation methods and mechanisms for FL in the literature, but they mainly focused on ensuring convergence under some assumptions such as data distribution, knowledge of malicious participant ratio, and update similarity of participants. Our work highlights the shortfall in current theoretical convergence guaranteed methods and presents a broader research goal to create aggregation mechanisms that work in harmony with Non-IID data, which is one of the key properties of FL.  

Our method is mainly based on outlier elimination which may tolerate up to a certain number of malicious participants in the system. As the ratio of attackers increases in the FL setting, they will have a high impact on the distribution of distances. Distance distributions also depend on some other parameters, such as the severity of the Non-IID data and coordination of attackers; therefore, to what extent ARFED can handle malicious participants is beyond the scope of this study and reserved as future work.

Moreover, ARFED allows a participant to be included aggregation step of the FL round if the participant is reliable at all layers with $and (\land)$ operation. The effect of losing participants on different layers will be examined for the more complex model architectures as future work.

Besides, MNIST, Fashion-MNIST, and CIFAR10 datasets were used in this study because they are widely used in federated learning and defense mechanism studies similar to the one we propose. However, the performance of our proposed method could be evaluated on different datasets in the future.



\appendix

\section{Appendix}\label{appendix}

\subsection{Model Architectures}\label{architecture}

The model architectures used for each datasets are shown in Table~\ref{tab:mnist_2nn_architecture} (MNIST-2NN \cite{mcmahan2017communication}), Table~\ref{tab:mnist_CNN_architecture} (MNIST-CNN \cite{mcmahan2017communication}), Table~\ref{tab:cifar_architecture} (CIFAR10 \cite{tolpegin2020data}) and Table~\ref{tab:fashionmnistarchitecture} (Fashion-MNIST). All activation functions for all models are ReLU.

\begin{table}
\centering
\caption{Model architecture of MNIST-2NN.}
\label{tab:mnist_2nn_architecture}
\begin{tabular}{cc}
\textbf{Layer}           & \textbf{Size}     \\ 
\hline
Fully Connected & (784, 200)  \\
Fully Connected & (200, 200)  \\
Fully Connected & (200, 10)  
\end{tabular}
\end{table}

\begin{table}
\centering
\caption{Model architecture of MNIST-CNN.}
\label{tab:mnist_CNN_architecture}
\begin{tabular}{cc} 
\textbf{Layer}          & \textbf{Size}      \\ 
\hline
Conv            & 32@5$\times$5    \\
Max Pooling     & 2$\times$2       \\
Conv            & 64@5$\times$5   \\
Max Pooling     & 2$\times$2       \\
Fully Connected & (1024, 512)  \\
Fully Connected & (512, 10)   
\end{tabular}
\end{table}

\begin{table}
\centering
\caption{Model architecture of Fashion-MNIST.}
\label{tab:fashionmnistarchitecture}
\begin{tabular}{cc}
\textbf{Layer}           & \textbf{Size}                \\ 
\hline
Conv            & 32@5$\times$5, pad=2   \\
Max Pooling     & 2$\times$2                 \\
Conv              & 64@5$\times$5, pad=2  \\
Max Pooling     & 2$\times$2                 \\
Fully Connected & (3136, 500)            \\
Fully Connected & (500, 10)             
\end{tabular}
\end{table}

\begin{table}[!hbp]
\centering
\caption{Model architecture of CIFAR10.}
\label{tab:cifar_architecture}
\begin{tabular}{cc}
\textbf{Layer}           & \textbf{Size}                    \\ 
\hline
Conv             & 32@3$\times$3, pad=1     \\
Conv             & 32@3$\times$3, pad=1    \\
Max Pooling     & 2$\times$2                     \\
Conv             & 64@3$\times$3, pad=1    \\
Conv             & 64@3$\times$3, pad=1    \\
Max Pooling     & 2$\times$2                     \\
Conv             & 128@3$\times$3, pad=1   \\
Conv             & 128@3$\times$3, pad=1  \\
Max Pooling     & 2$\times$2                     \\
Fully Connected & (2048, 128)                \\
Fully Connected & (128, 10)                 
\end{tabular}
\end{table}

The FL setting parameters used for each dataset are shown in Table~\ref{tab:parameters_tab}. Learning rate scheduling and clipping were not applied to MNIST and Fashion-MNIST, therefore related parameters are set as N/A (Not Applicable).

\begin{table}[ht!]
\caption{FL setting parameters used in experiments.}
\label{tab:parameters_tab}
\begin{adjustbox}{width=\columnwidth}
\begin{tabular}{lcccc}
\textbf{Parameters}              & \textbf{MNIST-2NN}  & \textbf{MNIST-CNN}  & \textbf{Fashion-MNIST} & \textbf{CIFAR10}    \\ 
\hline
number of participant (n)        & 100                 & 100                 & 100                    & 100                  \\
communication round (t)          & 200                 & 200                 & 200                    & 500                  \\
number of label in each~         & \multirow{2}{*}{2}  & \multirow{2}{*}{2}  & \multirow{2}{*}{2}     & \multirow{2}{*}{5}   \\
participant in IID setting       &                     &                     &                        &                      \\
number
  of label in each~       & \multirow{2}{*}{10} & \multirow{2}{*}{10} & \multirow{2}{*}{10}    & \multirow{2}{*}{10}  \\
participant
  in Non-IID setting &                     &                     &                        &                      \\
batch size                       & 32                  & 32                  & 25                     & 100                  \\
number of epoch                  & 10                  & 10                  & 10                     & 10                   \\
momentum                         & 0.9                 & 0.9                 & 0.9                    & 0.9                  \\
learning rate                    & 0.01                & 0.01                & 0.002                  & 0.0015               \\
minimum learning rate (min\_lr)  & N/A                 & N/A                 & N/A                    & 0.000010             \\
lr scheduler factor              & N/A                 & N/A                 & N/A                    & 0.2                  \\
best threshold                   & N/A                 & N/A                 & N/A                    & 0.0001               \\
clipping threshold               & N/A                 & N/A                 & N/A                    & 10                  
\end{tabular}
\end{adjustbox}
\end{table}


\subsection{Machine Configuration and Used Platforms}

Tesla P100-PCIE-16GB, Tesla V100-SXM2-16GB, and NVIDIA A100-SXM-80GB were used for the experiments. According to the used data sets, the average running time of an experiment on the A100 machine is as follows: 1 hour for MNIST-2NN, 1.4 hours for MNIST-CNN, 3.3 hours for FASHION-MNIST, and 7.6 hours for CIFAR.

Due to the versatility of the experimental settings, using an off-the-shelf platform did not provide the necessary flexibility; therefore, we chose to code ourselves. All implemented methods and designed experiments can be seen via \url{https://github.com/eceisik/ARFED}. The required packages for the environment setup are also listed here.

\subsection{Experiments}\label{appendix:experiments}

In order to evaluate the performance of ARFED, extensive experiments covering different scenarios such as whether the attackers are organized, different types of attacks, the effect of the data distribution, and malicious participant ratio have been carried out on different datasets with different model architectures.

The experimental results of
\begin{itemize}
    \item the label flipping attacks are summarized for IID setting in Table~\ref{tab:label_flipping_iid_table} and for Non-IID setting in Table~\ref{tab:label_flipping_noniid_table}
    \item the Byzantine attacks for IID setting in Table~\ref{tab:byz_iid_table} and for Non-IID setting in Table~\ref{tab:byz_noniid_table} and
    \item the partial knowledge attacks for IID setting in Table~\ref{tab:fang_iid_table} and for Non-IID setting in Table~\ref{tab:fang_noniid_table}
\end{itemize}
    
for each data set under Section~\ref{experimental_results}. Due to the space limitation and to improve the readability of this study, the figures of the experiments of MNIST CNN, Fashion-MNIST, and CIFAR10 are presented here.

The figures reveal that the experiments run for different datasets and model architectures are in line with previous findings and are valid for all data sets.
For example, the performance loss in Byzantine attacks is greatest, and partial knowledge attacks cause more performance degradation than label flipping attacks. When the data distributions of the participants are Non-IID, the performance degrades more compared to IID cases. As the ratio of malicious participants increases, performance degrades more. Organized attackers cause more performance degradation. The worst accuracy score is recorded when the attackers are organized and the data distribution of the participants is Non-IID.

ARFED could eliminate the harmful effects of all attack types for both IID and non-IID cases and achieve accuracy scores close to when all collaborators are trusted cases (no attack case). On the other hand, CwMedian is not able to handle the attacks in the Non-IID setting. CwMedian could tolerate the performance degradation in only IID cases. For non-IID cases, it could show only a slight improvement or worsen the performance degradation. ARFED generally performs better than TrimmedMean in label flipping attacks and partial knowledge attacks, but they get similar scores in Byzantine attacks. However, it is worth remembering that TrimmedMean requires information of the malicious collaborator ratio in the system, while ARFED does not make such an assumption.

\subsubsection{MNIST CNN Experiments}

\begin{itemize}
    \item Figure~\ref{fig:mnist_cnn_lf_fig} presents the results of the experiments carried out for label flipping attack
    \item Figure~\ref{fig:mnist_cnn_byz_fig} present the results of the experiments carried out for byzantine attack and 
    \item Figure~\ref{fig:mnist_cnn_fang_fig} present the results of the experiments carried out for adaptive partial knowledge attack 
\end{itemize}
on MNIST dataset with CNN model architecture. 

\begin{figure}[h!]
  \centering
  \includegraphics[width=\columnwidth]{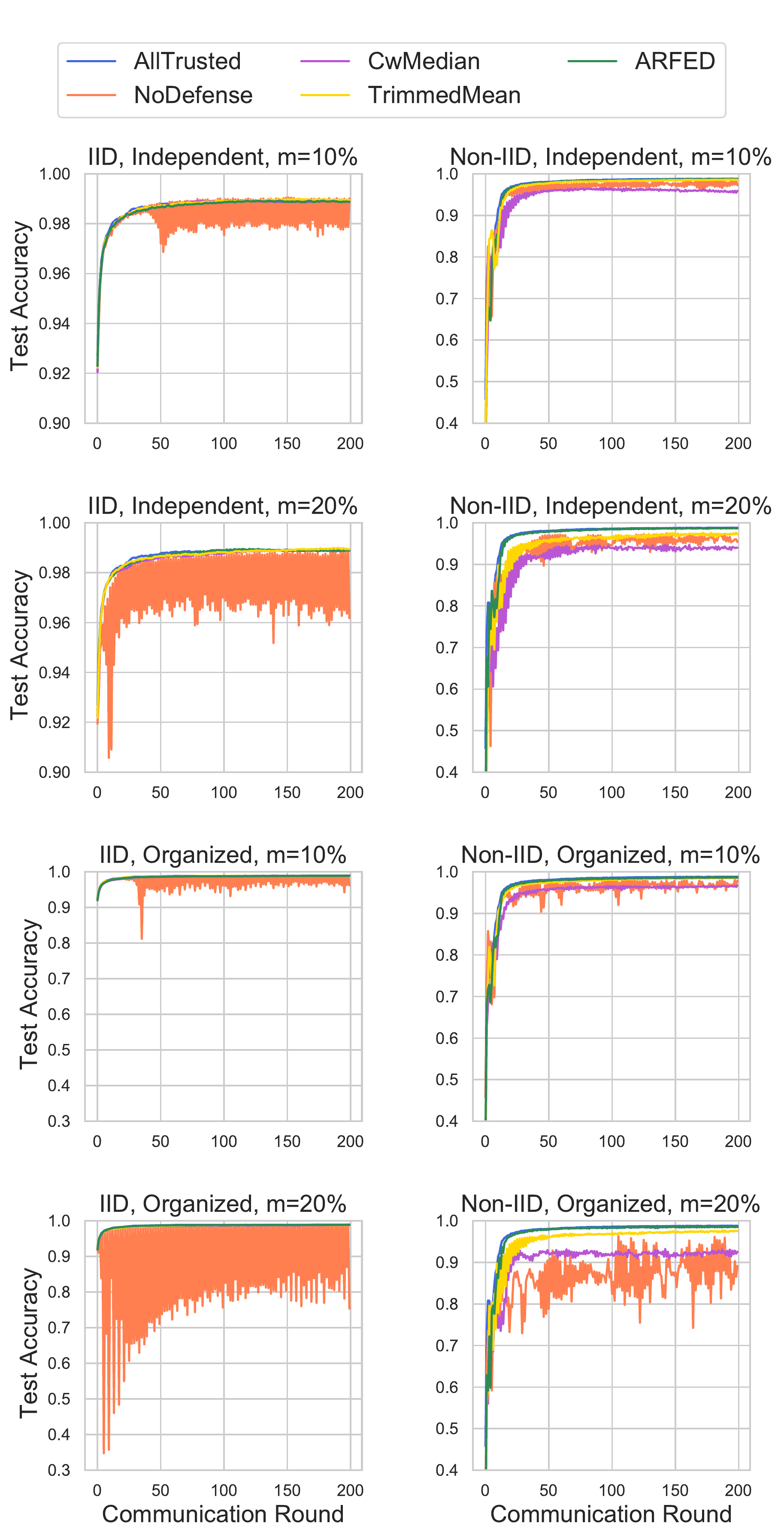}
  \caption{Accuracy curves of different strategies for MNIST CNN under label flipping attacks at different attacker ratios.}
  \label{fig:mnist_cnn_lf_fig}
\end{figure}

\begin{figure}[h!]
  \centering
  \includegraphics[width=\columnwidth]{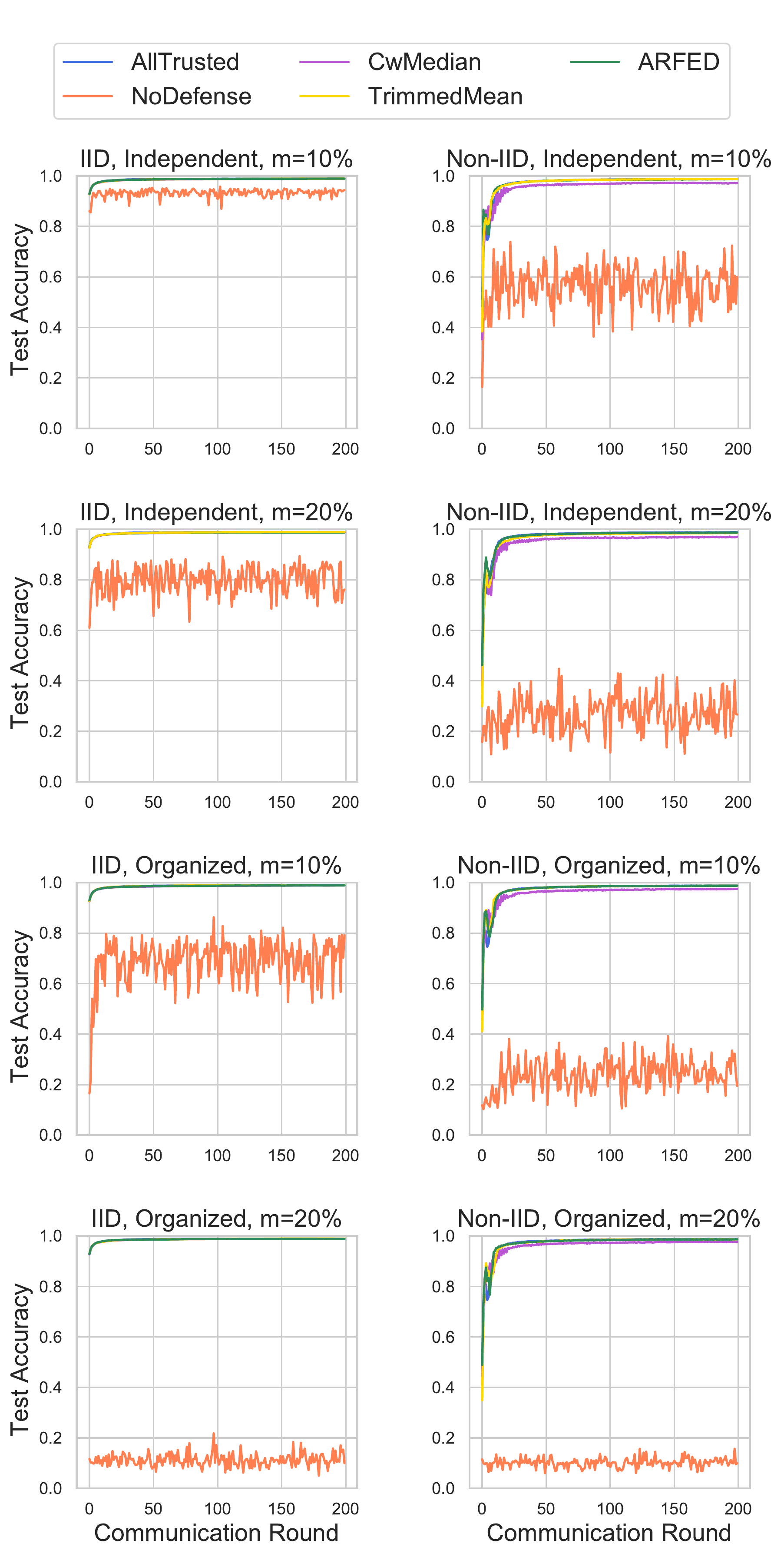}
  \caption{Accuracy curves of different strategies for MNIST CNN under Byzantine attacks at different attacker ratios.}
  \label{fig:mnist_cnn_byz_fig}
\end{figure}

\begin{figure}[h!]
  \centering
  \includegraphics[width=\columnwidth]{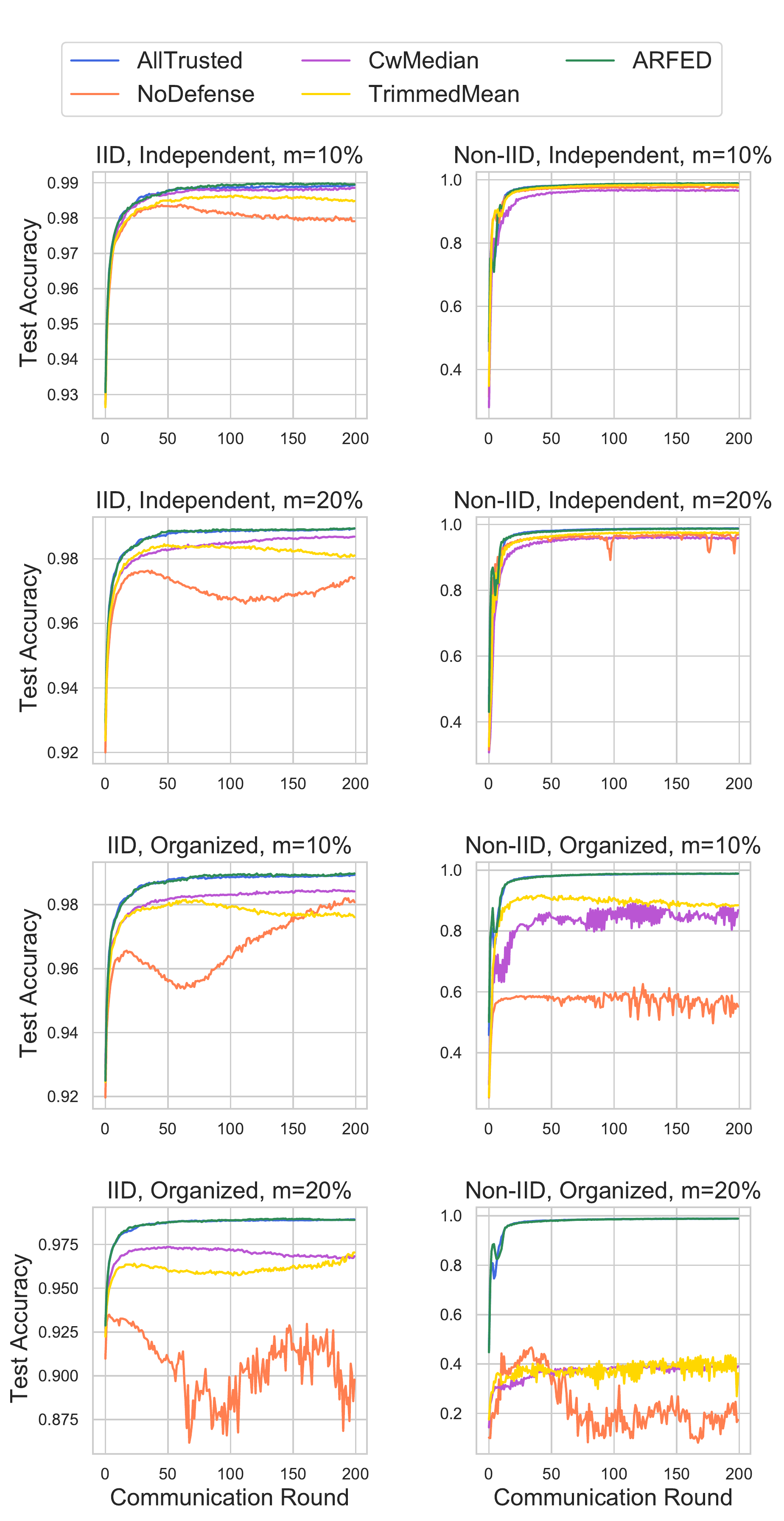}
  \caption{Accuracy curves of different strategies for MNIST CNN under adaptive partial knowledge attacks at different attacker ratios.}
  \label{fig:mnist_cnn_fang_fig}
\end{figure}

\pagebreak
\subsubsection{Fashion-MNIST Experiments}

\begin{itemize}
    \item Figure~\ref{fig:fashion_lf_fig} presents the results of the experiments carried out for label flipping attack
    \item Figure~\ref{fig:fashion_byz_fig} present the results of the experiments carried out for byzantine attack and 
    \item Figure~\ref{fig:fashion_fang_fig} present the results of the experiments carried out for adaptive partial knowledge attack
\end{itemize}
on Fashion-MNIST.

\begin{figure}[t!]
  \centering
  \includegraphics[width=\columnwidth]{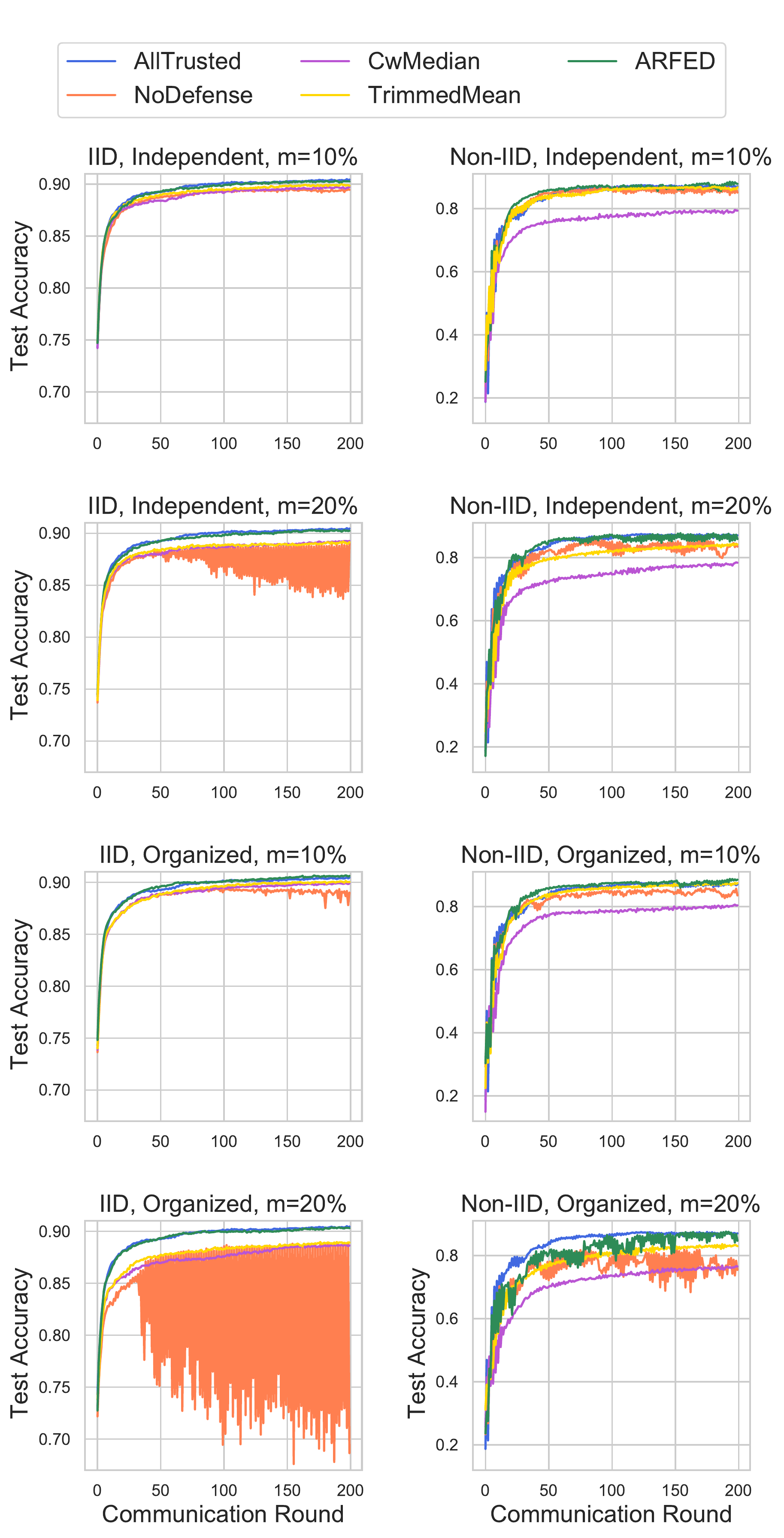}
  \caption{Accuracy curves of different strategies for Fashion-MNIST under label flipping attacks at different attacker ratios.}
  \label{fig:fashion_lf_fig}
\end{figure}

\begin{figure}[t!]
  \centering
  \includegraphics[width=\columnwidth]{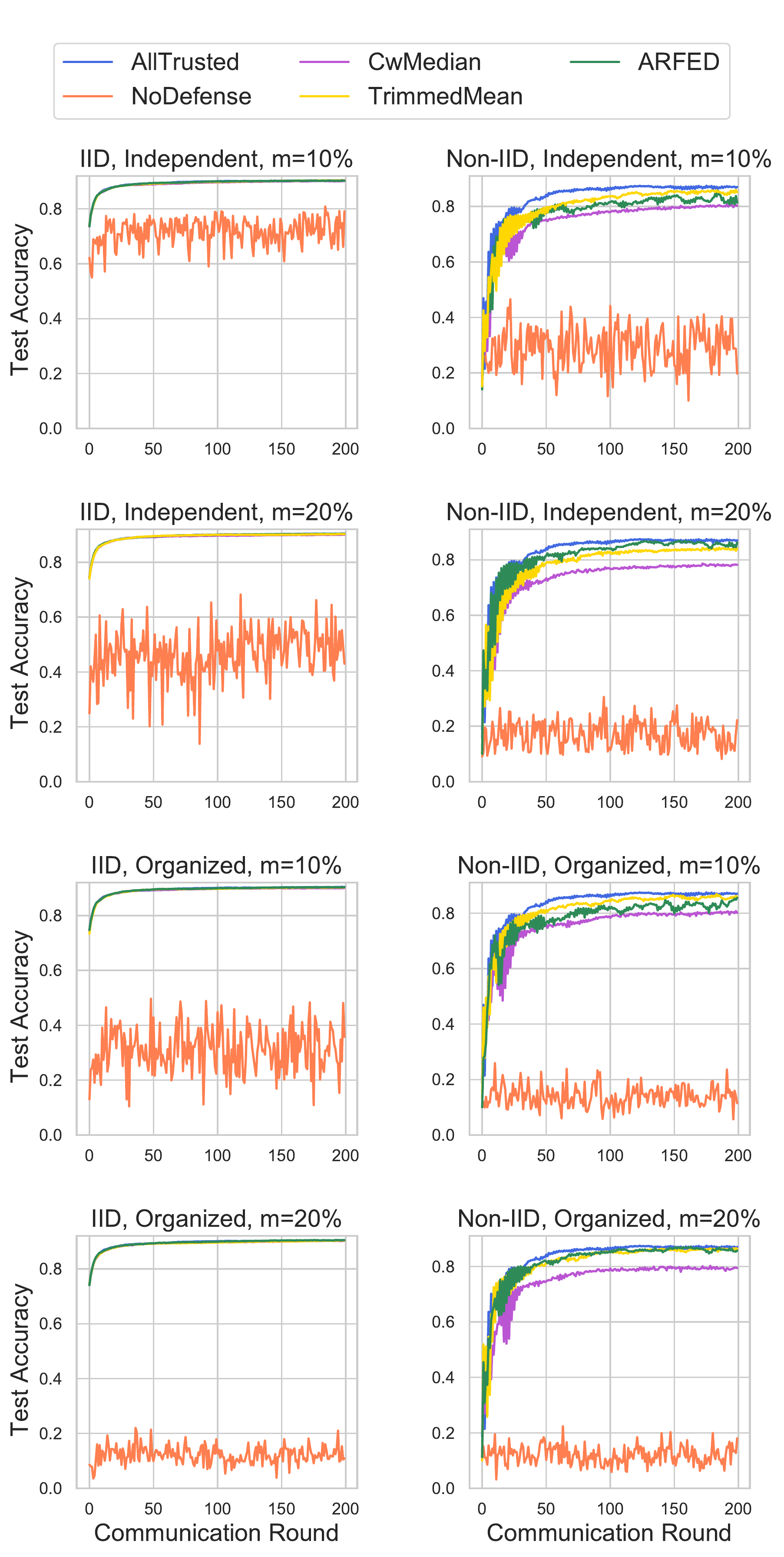}
  \caption{Accuracy curves of different strategies for Fashion-MNIST under Byzantine attacks at different attacker ratios.}
  \label{fig:fashion_byz_fig}
\end{figure}

\begin{figure}[htbp!]
  \centering
  \includegraphics[width=\columnwidth]{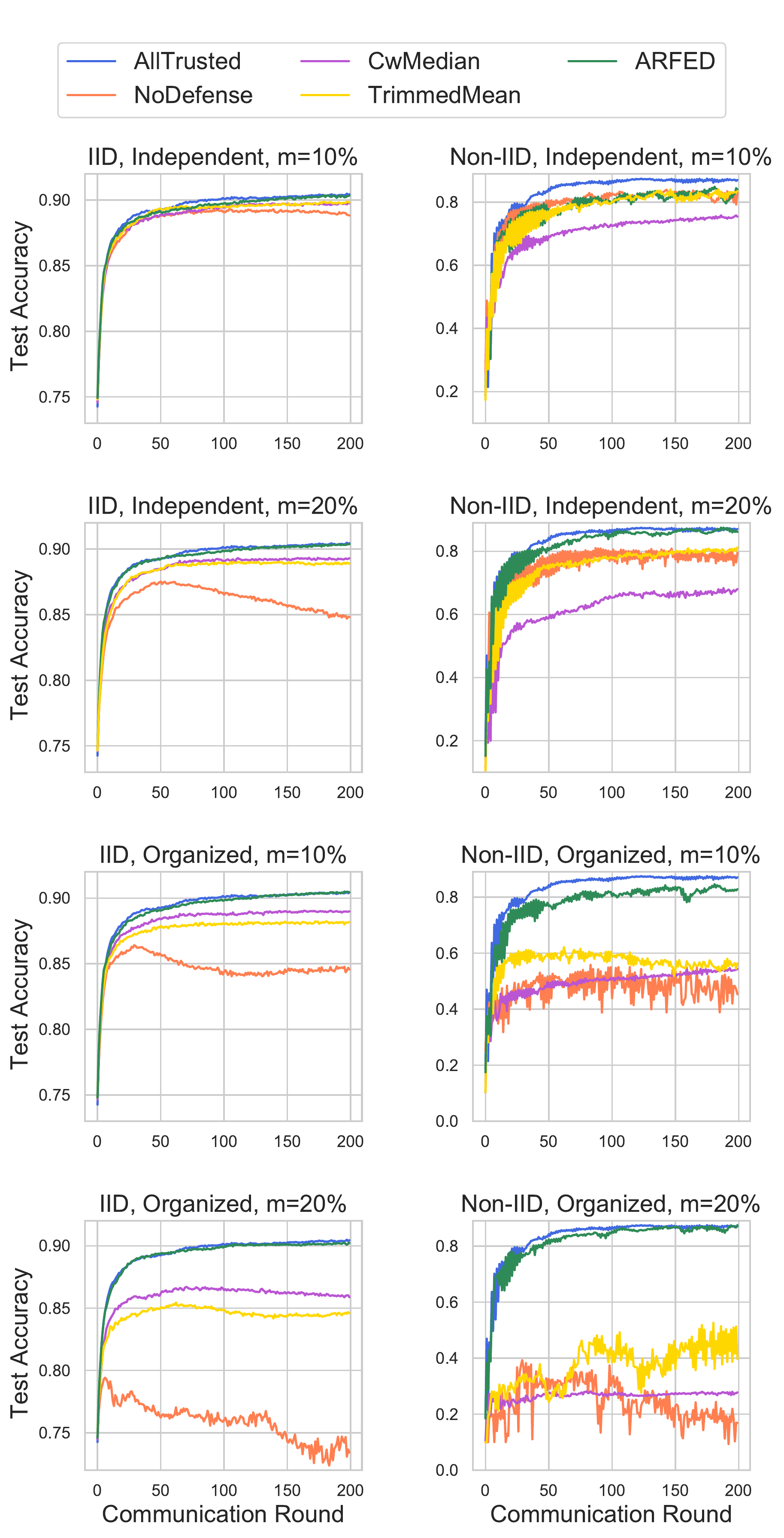}
  \caption{Accuracy curves of different strategies for Fashion-MNIST under adaptive partial knowledge attacks at different attacker ratios.}
  \label{fig:fashion_fang_fig}
\end{figure}

\subsubsection{CIFAR10 Experiments}

\begin{itemize}
    \item Figure~\ref{fig:cifar_lf_fig} presents the results of the experiments carried out for label flipping attack
    \item Figure~\ref{fig:cifar_byz_fig} present the results of the experiments carried out for byzantine attack and
    \item Figure~\ref{fig:cifar_fang_fig} present the results of the experiments carried out for adaptive partial knowledge attack 
\end{itemize}
on CIFAR10.

\begin{figure}[ht!]
  \centering
  \includegraphics[width=\columnwidth]{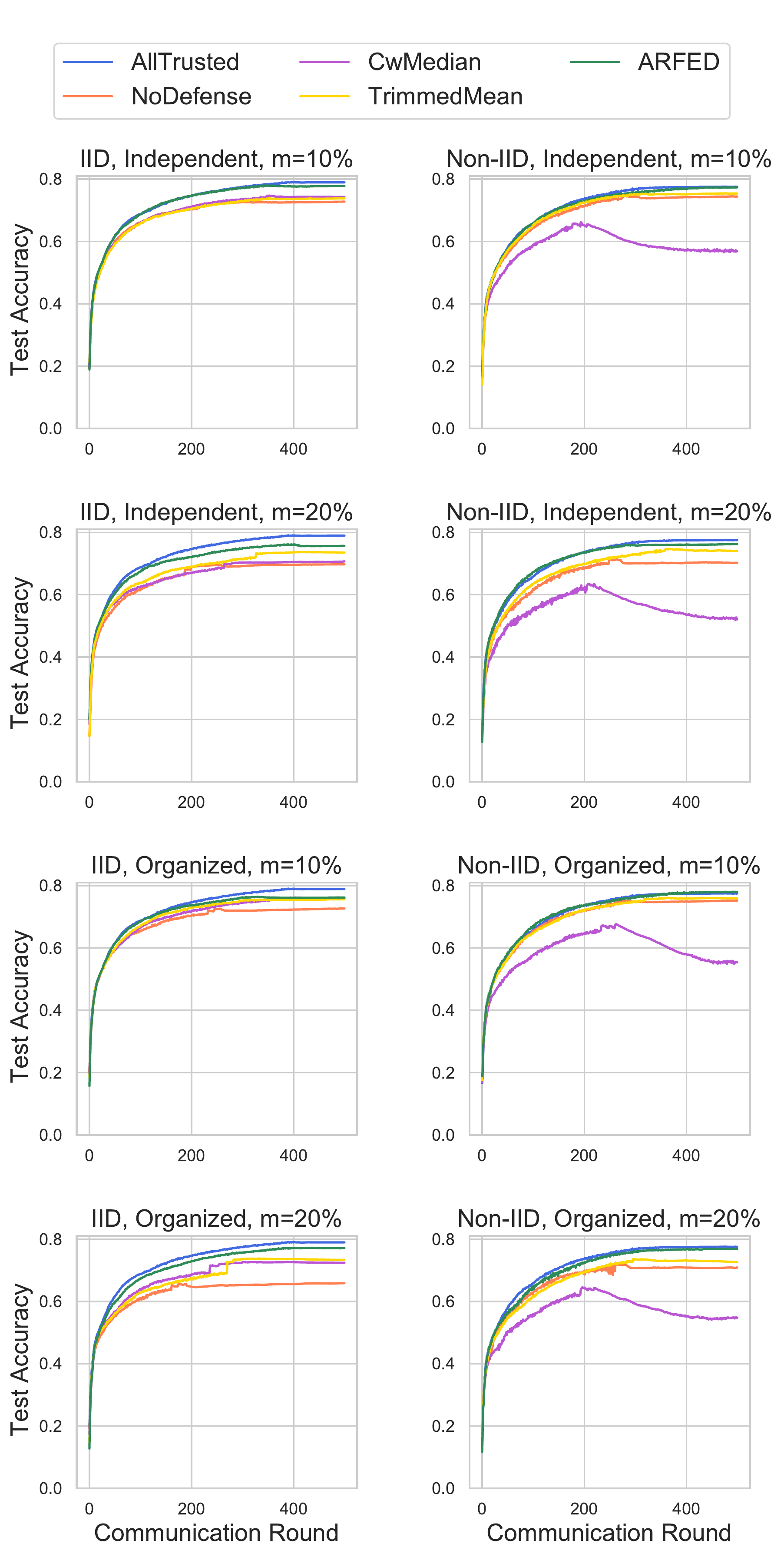}
  \caption{Accuracy curves of different strategies for CIFAR10 under label flipping attacks at different attacker ratios.}
  \label{fig:cifar_lf_fig}
\end{figure}

\begin{figure}[ht!]
  \centering
  \includegraphics[width=\columnwidth]{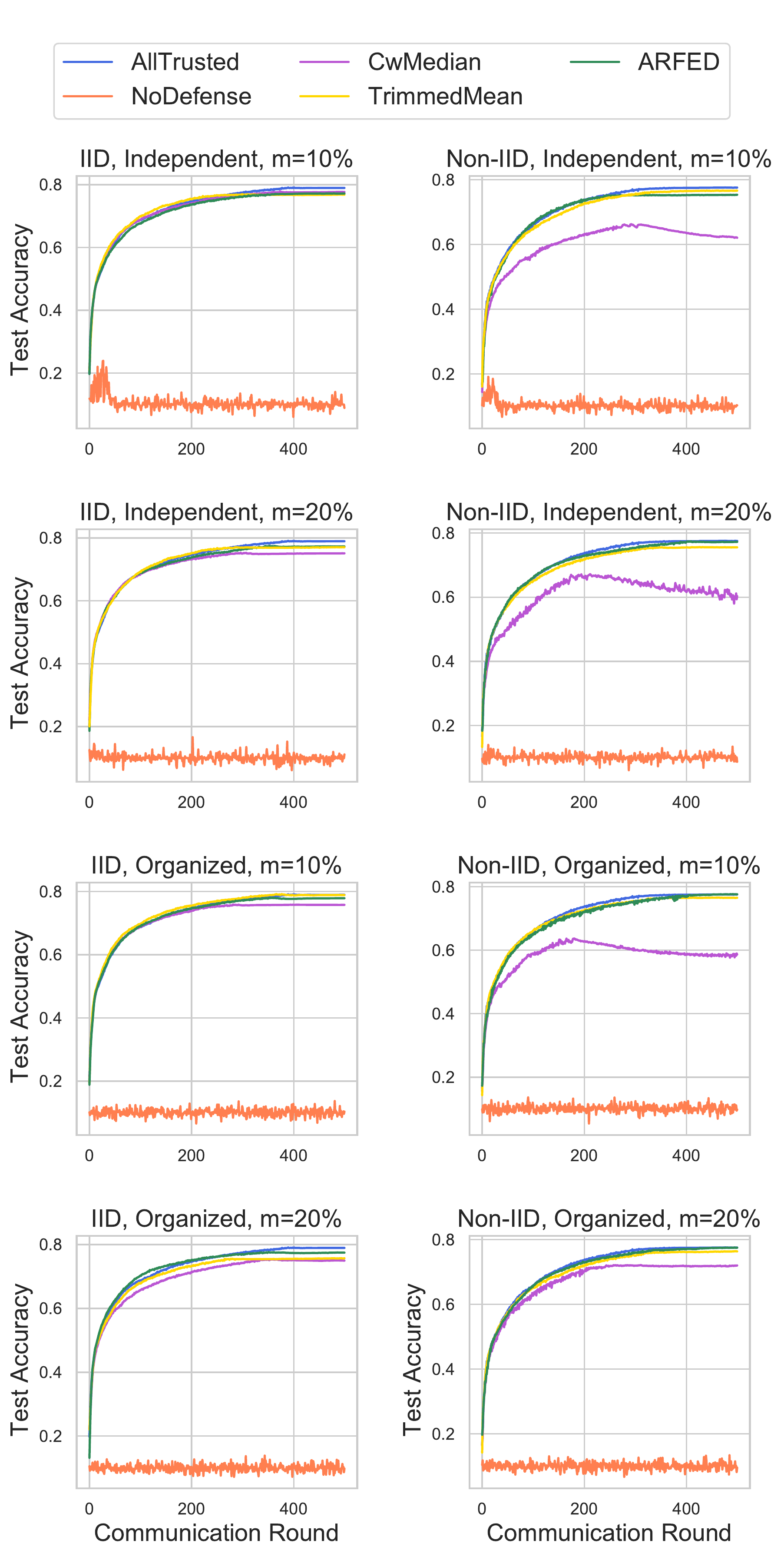}
  \caption{Accuracy curves of different strategies for CIFAR10 under Byzantine attacks at different attacker ratios.}
  \label{fig:cifar_byz_fig}
\end{figure}

\begin{figure}[ht!]
  \centering
  \includegraphics[width=\columnwidth]{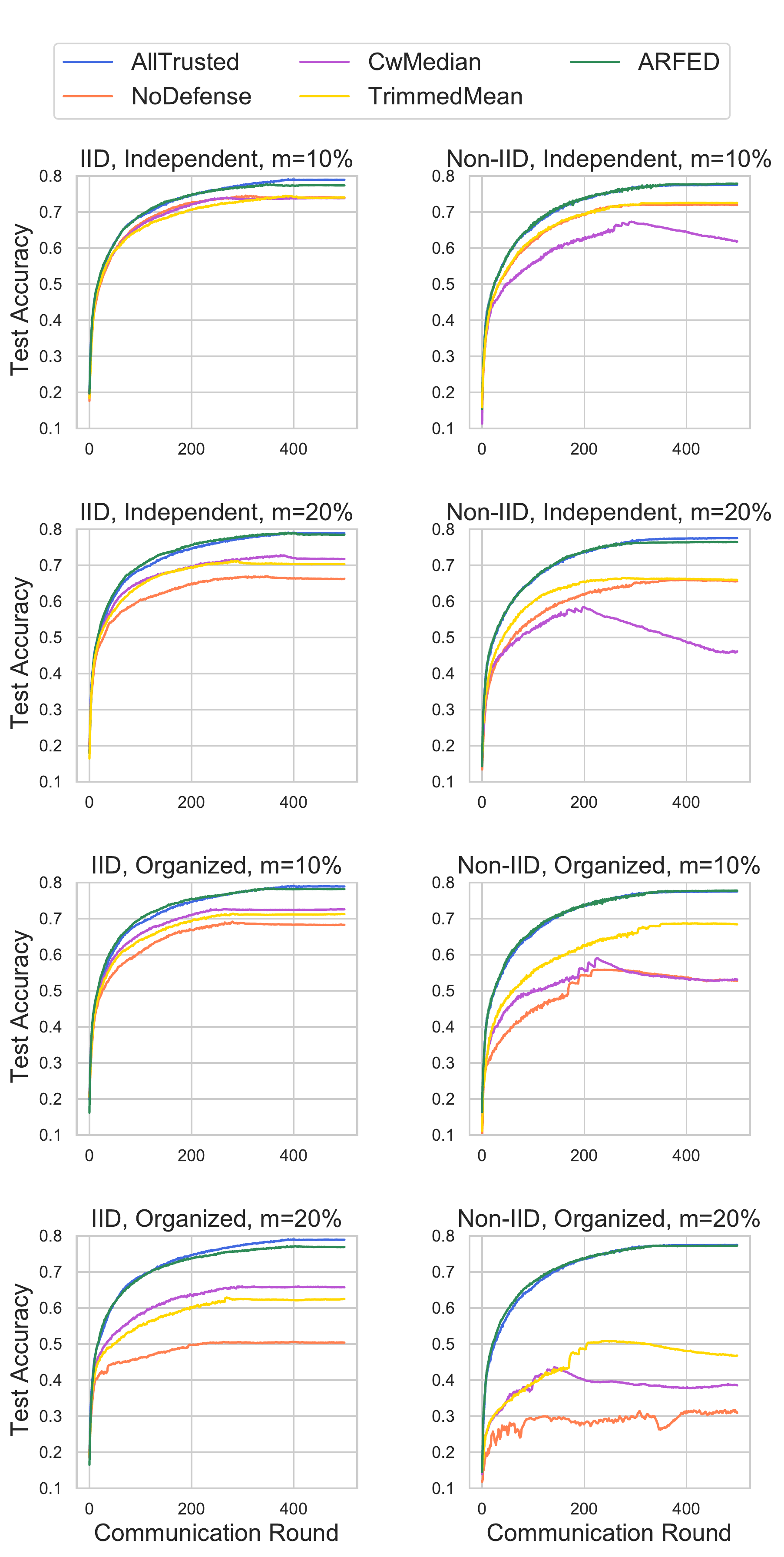}
  \caption{Accuracy curves of different strategies for CIFAR10 under adaptive partial knowledge attacks at different attacker ratios.}
  \label{fig:cifar_fang_fig}
\end{figure}

\subsubsection{Number of Reliable and Outlier Participants for CIFAR10 Experiments }

The working mechanism of ARFED decides which participant is eligible to be included in the aggregation step based on whether the parameters sent by the participant for each layer of the model architecture are in the safe interval. This all-or-nothing approach of ARFED may raise concerns that too many participants may be discarded from the aggregation step, and too much valuable information may be lost through the layers. The model architecture used for the CIFAR10 dataset has more layers than the architectures that are used for other datasets. For this reason, the risk of losing too many participants brought by the algorithm's $and (\land)$ operation is expected to be best observed in the CIFAR10 set.

Figure~\ref{fig:cifar_lf_num_of_outliers_fig} illustrates the number of participants marked as reliable and included aggregation step versus the number of participants marked as outliers and discarded from aggregation in label flipping attacks with different scenarios. Figure~\ref{fig:cifar_byz_num_of_outliers_fig} illustrates the number of participants marked as reliable and included aggregation step versus the number of participants marked as outliers and discarded from aggregation in Byzantine attacks with different scenarios. The number of discarded participants is in line with the malicious participant ratio.

\begin{figure}[ht!]
  \centering
  \includegraphics[width=\columnwidth]{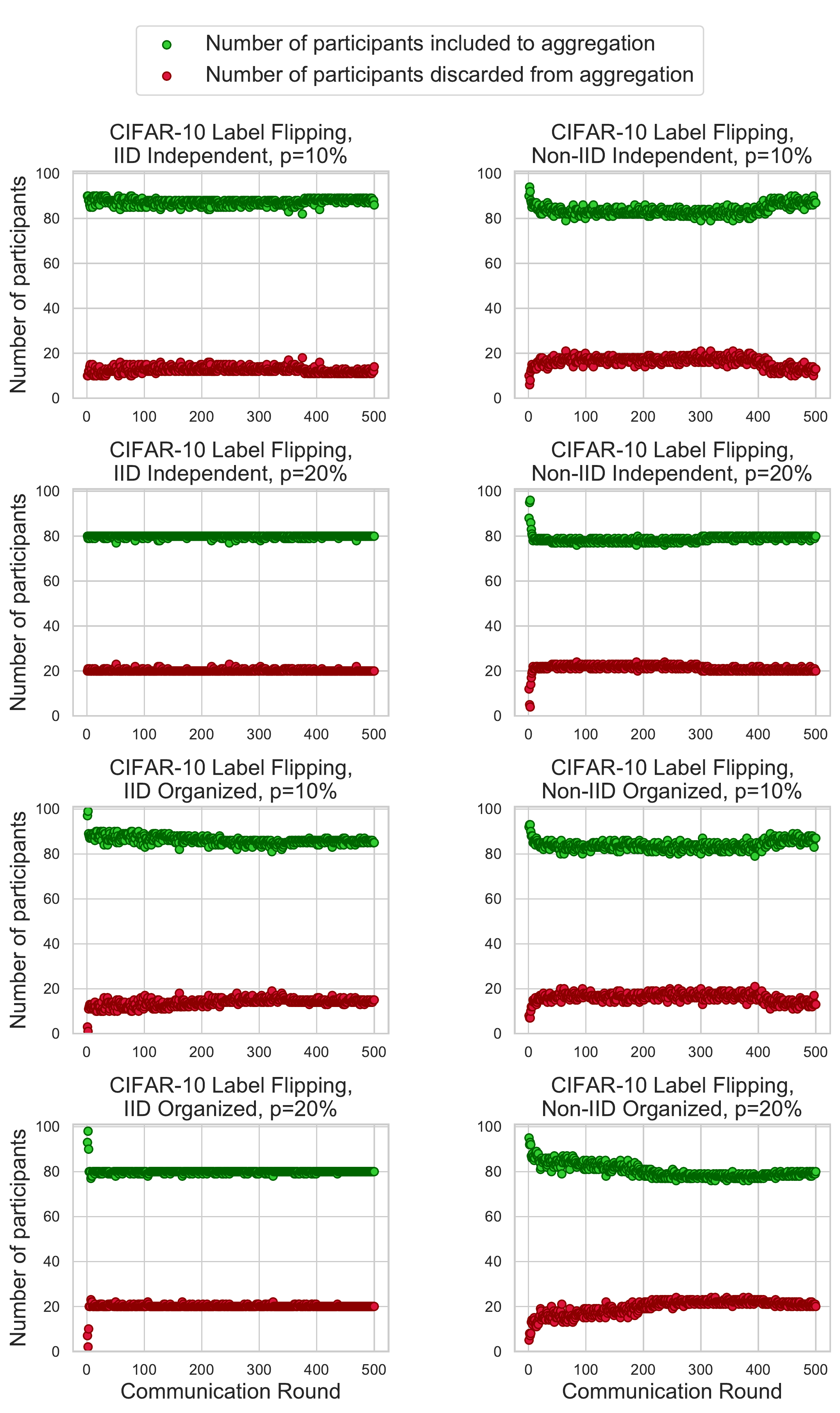}
  \caption{Number of participants marked as reliable and outlier in CIFAR10 label flipping attacks.}
  \label{fig:cifar_lf_num_of_outliers_fig}
\end{figure}

\begin{figure}[ht!]
  \centering
  \includegraphics[width=\columnwidth]{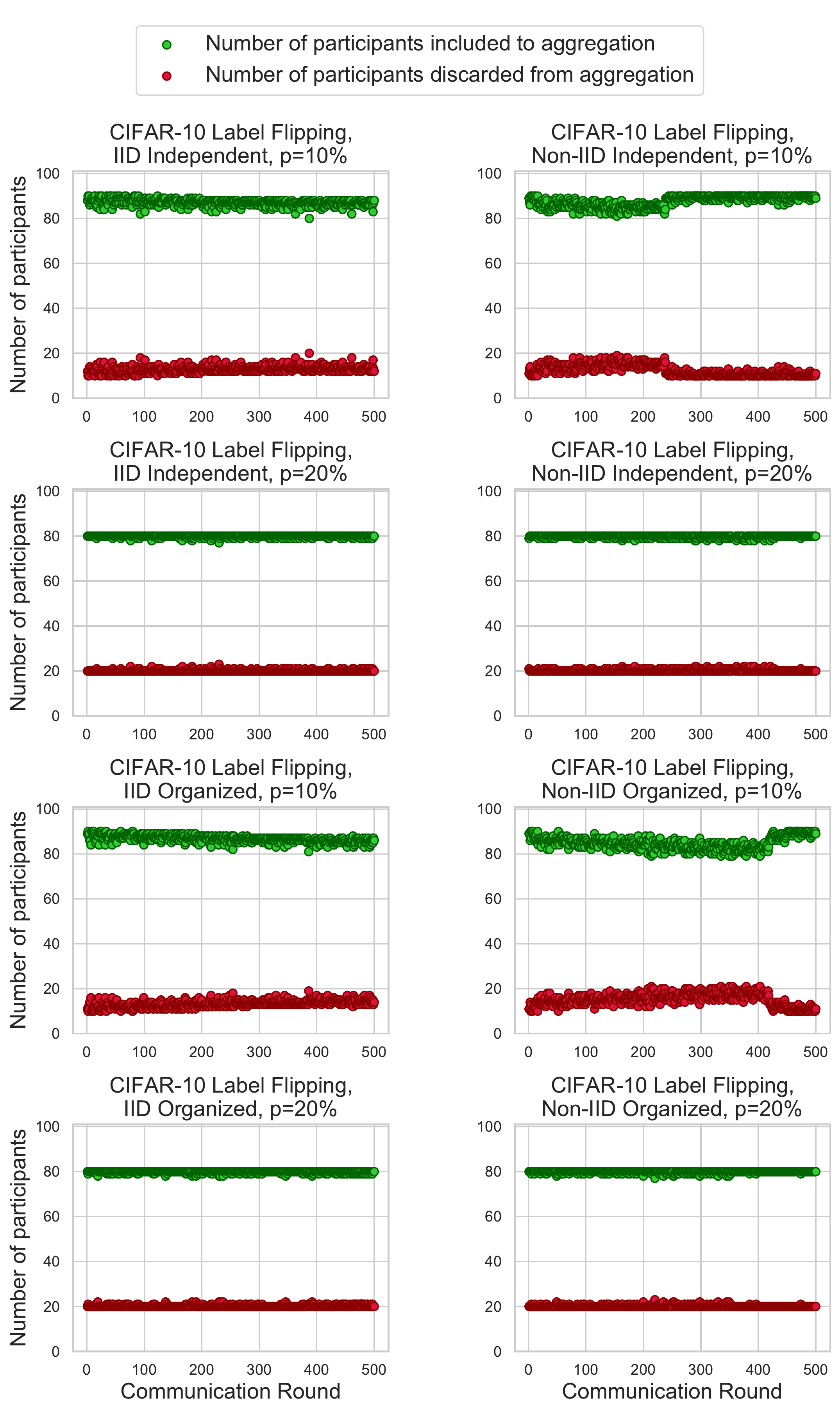}
  \caption{Number of participants marked as reliable and outlier in CIFAR10 Byzantine attacks.}
  \label{fig:cifar_byz_num_of_outliers_fig}
\end{figure}

\subsection{Box Plot Factor Comparison}\label{boxplotcomparison}

Different factor values were tested to show the performance impact of how strict the algorithm is in labeling a participant as reliable. For this reason, different factor values were applied when determining the lower distance threshold ($min\_d_l^t$) and upper distance threshold ($max\_d_l^t$) in Lines 19-20 of the Algorithm~\ref{alg:defensealgorithm}.

Table~\ref{tab:fang_valid_accuracy_iid_table} and Figure~\ref{fig:mnist_fang_factor_comparison_iid} present the results of the partial knowledge attack scenarios when the participants' data are IID. Table~\ref{tab:fang_valid_accuracy_noniid_table}  and Figure~\ref{fig:mnist_fang_factor_comparison_noniid} show the results of the partial knowledge attack scenarios when the participants' data are non-IID. These graphs show no significant difference between different factor values, i.e., 0, 1, 1.5, and 2, especially in the IID setting. However, in the non-IID setting, although f=0 achieves performance like others, the accuracy graph has oscillations that can signal a convergence problem. One possible reason might be that the data was non-IID, and the algorithm could not obtain a good enough sample space by eliminating too many participants from the main model aggregation step.

\begin{table}[ht!]
\caption{Accuracy scores obtained on test set under partial knowledge attacks with different malicious participant ratios in the IID setting. The best results are bold.}
\label{tab:fang_valid_accuracy_iid_table}
\begin{adjustbox}{width=\columnwidth,center}
\begin{tabular}{cccccccccc}
                      & \multicolumn{4}{c}{\textbf{Organized~}}                       & \textbf{~} & \multicolumn{4}{c}{\textbf{Independent}}                       \\ 
\hline\hline
                      & \multicolumn{2}{c}{m=10\%}    & \multicolumn{2}{c}{m=20\%}    & ~          & \multicolumn{2}{c}{m=10\%}    & \multicolumn{2}{c}{m=20\%}     \\
\textbf{~}            & min           & max           & min           & max           & ~          & min           & max           & min           & max            \\ 
\hline\hline
\textbf{NoDefense}                               & 60.4          & 87.6          & 11.3          & 12.8          &   & 96.4          & 96.5          & 94.8          & 95.0           \\
\textbf{CwMedian}                                   & 97.4          & 97.5          & 94.7          & 94.8          &   & 98.2          & 98.2          & 97.5          & 97.5           \\
\textbf{TrimmedMean}                             & 96.4          & 96.5          & 89.0          & 94.9          &   & 98.3          & 98.3          & 97.1          & 97.1           \\
\textbf{ARFED f1.5}                             & \textbf{99.0} & \textbf{99.0} & \textbf{98.9} & \textbf{98.9} &   & \textbf{99.0} & \textbf{99.0} & \textbf{98.9} & \textbf{98.9}  \\
\rowcolor[rgb]{1,0.843,0.749} \textbf{ARFED f0} & 98.7          & 98.7          & 98.8          & 98.8          &   & 98.7          & 98.7          & 98.8          & 98.8           \\
\rowcolor[rgb]{1,0.843,0.749}\textbf{ARFED f1}                               & 98.9          & 99.0          & \textbf{98.9} & \textbf{98.9} &   & 98.9          & \textbf{99.0} & \textbf{98.9} & \textbf{98.9}  \\
\rowcolor[rgb]{1,0.843,0.749}\textbf{ARFED f2}                               & \textbf{99.0} & \textbf{99.0} & \textbf{98.9} & \textbf{98.9} &   & \textbf{99.0} & \textbf{99.0} & \textbf{98.9} & \textbf{98.9}
\end{tabular}
\end{adjustbox}
\end{table}

\begin{figure}[ht!]
  \centering
  \includegraphics[width=\linewidth]{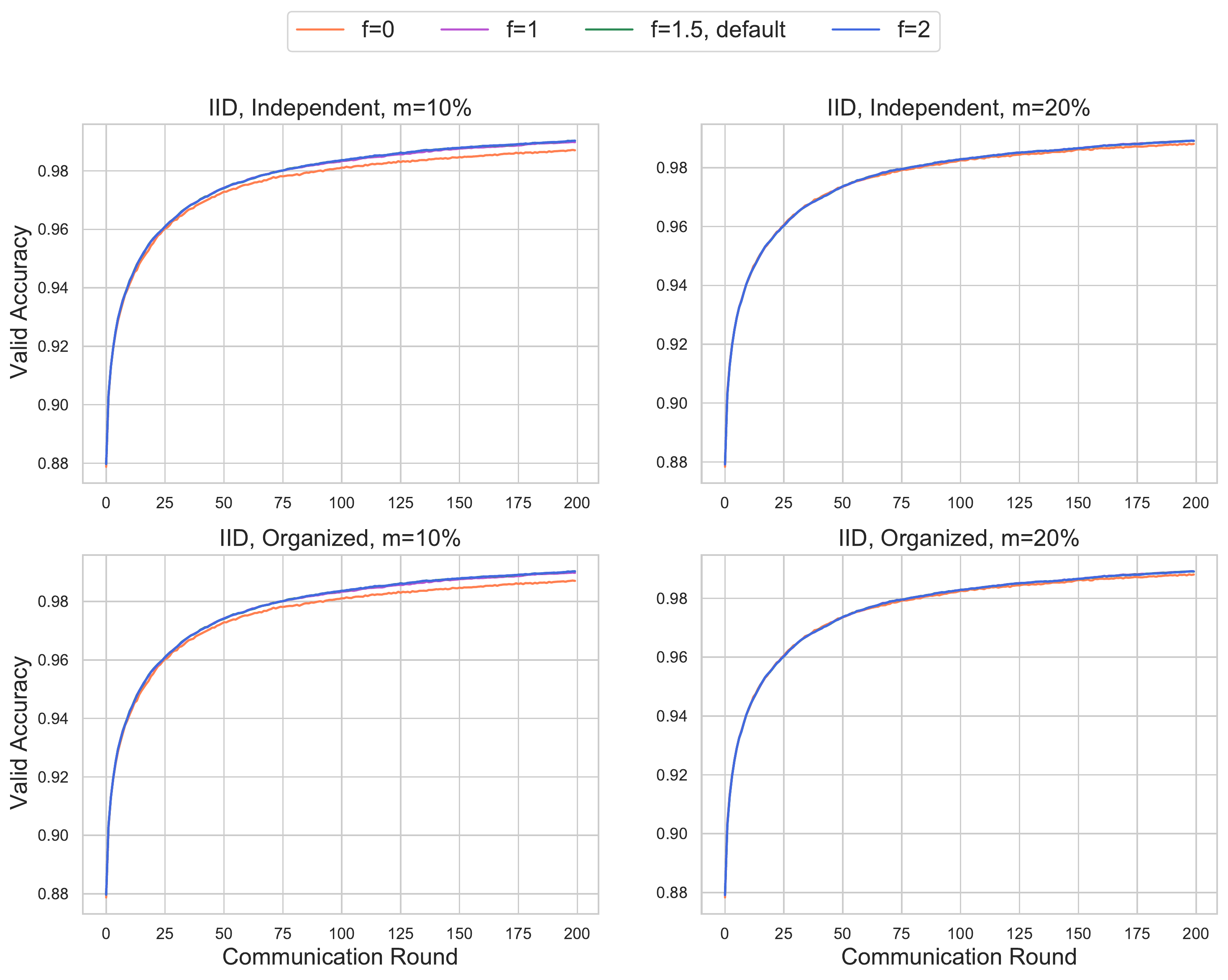}
  \caption{Box plot factor comparison on accuracy for IID cases }
  \label{fig:mnist_fang_factor_comparison_iid}
\end{figure}

\begin{table}[ht!]
\caption{Accuracy scores obtained on test set under partial knowledge attacks with different malicious participant ratios in the Non-IID setting. The best results are bold.}
\label{tab:fang_valid_accuracy_noniid_table}
\begin{adjustbox}{width=\columnwidth,center}
\begin{tabular}{cccccccccc}
                      & \multicolumn{4}{c}{\textbf{Organized~}}                       & \textbf{~} & \multicolumn{4}{c}{\textbf{Independent}}                       \\ 
\hline\hline
                      & \multicolumn{2}{c}{m=10\%}    & \multicolumn{2}{c}{m=20\%}    & ~          & \multicolumn{2}{c}{m=10\%}    & \multicolumn{2}{c}{m=20\%}     \\
\textbf{~}            & min           & max           & min           & max           & ~          & min           & max           & min           & max            \\ 
\hline\hline
\textbf{NoDefense}   & 51.1          & 60.4          & 13.2          & 16.7          &   & 88.5          & 92.9          & 85.6          & 91.0           \\
\textbf{CwMedian}       & 62.8          & 67.0          & 21.8          & 35.6          &   & 78.4          & 84.8          & 63.6          & 79.6           \\
\textbf{TrimmedMean} & 77.2          & 81.9          & 31.6          & 38.4          &   & 94.8          & 95.4          & 91.5          & 92.7           \\
\textbf{ARFED f1.5}       & \textbf{97.1} & \textbf{97.2} & \textbf{96.8} & 96.9          &   & \textbf{97.1} & \textbf{97.2} & \textbf{96.8} & 96.9           \\
\rowcolor[rgb]{1,0.843,0.749}\textbf{ARFED f0}   & 96.4          & \textbf{97.2} & 96.2          & \textbf{97.3} &   & 96.4          & \textbf{97.2} & 96.2          & \textbf{97.3}  \\
\rowcolor[rgb]{1,0.843,0.749}\textbf{ARFED f1}   & 97.0          & 97.1          & 96.6          & 96.9          &   & 97.0          & 97.1          & 96.6          & 96.9           \\
\rowcolor[rgb]{1,0.843,0.749}\textbf{ARFED f2}   & \textbf{97.1} & \textbf{97.2} & \textbf{96.8} & 96.9          &   & \textbf{97.1} & \textbf{97.2} & \textbf{96.8} & 96.9
\end{tabular}
\end{adjustbox}
\end{table}

\begin{figure}[ht!]
  \centering
  \includegraphics[width=\linewidth]{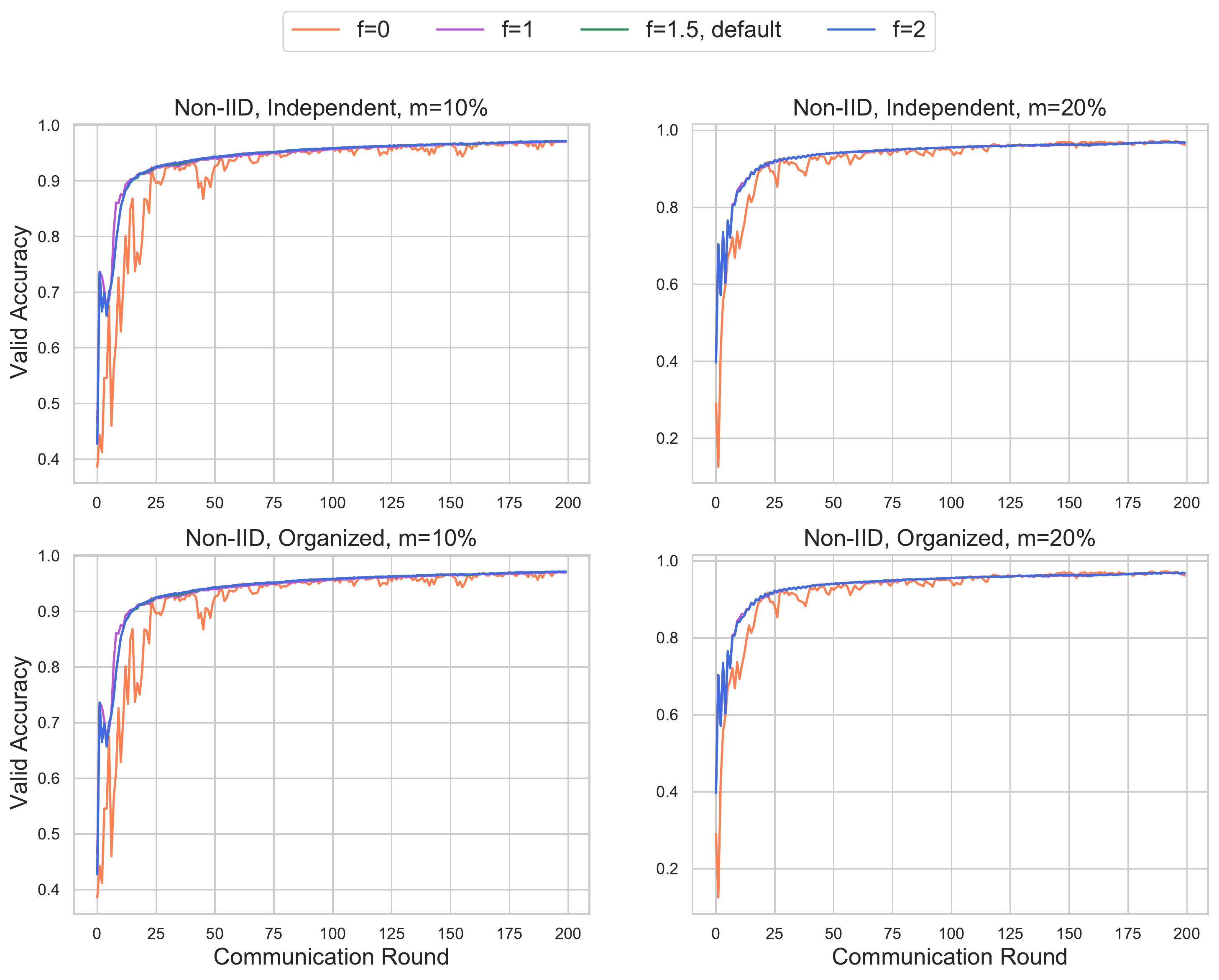}
  \caption{Box plot factor comparison on accuracy for non-IID cases }
  \label{fig:mnist_fang_factor_comparison_noniid}
\end{figure}




\clearpage
\bibliographystyle{elsarticle-num}
\bibliography{references}

\begin{thebibliography}{10}
\expandafter\ifx\csname url\endcsname\relax
  \def\url#1{\texttt{#1}}\fi
\expandafter\ifx\csname urlprefix\endcsname\relax\def\urlprefix{URL }\fi
\expandafter\ifx\csname href\endcsname\relax
  \def\href#1#2{#2} \def\path#1{#1}\fi

\bibitem{fang2020local}
M.~Fang, X.~Cao, J.~Jia, N.~Gong,
  \href{https://www.usenix.org/conference/usenixsecurity20/presentation/fang}{Local
  model poisoning attacks to {Byzantine-Robust} federated learning}, in: 29th
  USENIX Security Symposium (USENIX Security 20), USENIX Association, 2020, pp.
  1605--1622.
\newline\urlprefix\url{https://www.usenix.org/conference/usenixsecurity20/presentation/fang}

\bibitem{altan_a_osmundsen2018digital}
K.~Osmundsen, J.~Iden, B.~Bygstad, Digital transformation: Drivers, success
  factors, and implications., in: MCIS, 2018, p.~37.

\bibitem{altan_b_9202502}
C.~Naseeb, Ai and ml-driving and exponentiating sustainable and quantifiable
  digital transformation, in: 2020 IEEE 44th Annual Computers, Software, and
  Applications Conference (COMPSAC), 2020, pp. 316--321.
\newblock \href {https://doi.org/10.1109/COMPSAC48688.2020.0-227}
  {\path{doi:10.1109/COMPSAC48688.2020.0-227}}.

\bibitem{altan_c_SESTINO2020102173}
A.~Sestino, M.~I. Prete, L.~Piper, G.~Guido,
  \href{https://www.sciencedirect.com/science/article/pii/S0166497220300456}{Internet
  of things and big data as enablers for business digitalization strategies},
  Technovation 98 (2020) 102173.
\newblock \href
  {https://doi.org/https://doi.org/10.1016/j.technovation.2020.102173}
  {\path{doi:https://doi.org/10.1016/j.technovation.2020.102173}}.
\newline\urlprefix\url{https://www.sciencedirect.com/science/article/pii/S0166497220300456}

\bibitem{altan_d_6817512}
X.-W. Chen, X.~Lin, Big data deep learning: Challenges and perspectives, IEEE
  Access 2 (2014) 514--525.
\newblock \href {https://doi.org/10.1109/ACCESS.2014.2325029}
  {\path{doi:10.1109/ACCESS.2014.2325029}}.

\bibitem{altan_e_wang2020big}
J.~Wang, Y.~Yang, T.~Wang, R.~S. Sherratt, J.~Zhang, Big data service
  architecture: a survey, Journal of Internet Technology 21~(2) (2020)
  393--405.

\bibitem{altan_f_8622621}
N.~Gruschka, V.~Mavroeidis, K.~Vishi, M.~Jensen, Privacy issues and data
  protection in big data: A case study analysis under gdpr, in: 2018 IEEE
  International Conference on Big Data (Big Data), 2018, pp. 5027--5033.
\newblock \href {https://doi.org/10.1109/BigData.2018.8622621}
  {\path{doi:10.1109/BigData.2018.8622621}}.

\bibitem{custers2019eu}
B.~Custers, A.~M. Sears, F.~Dechesne, I.~Georgieva, T.~Tani, S.~Van~der Hof, EU
  personal data protection in policy and practice, Vol.~29, Springer, 2019.

\bibitem{voigt2017eu}
P.~Voigt, A.~Von~dem Bussche, The eu general data protection regulation (gdpr),
  A Practical Guide, 1st Ed., Cham: Springer International Publishing 10 (2017)
  3152676.

\bibitem{mcmahan2017communication}
B.~McMahan, E.~Moore, D.~Ramage, S.~Hampson, B.~A. y~Arcas,
  Communication-efficient learning of deep networks from decentralized data,
  in: Artificial Intelligence and Statistics, PMLR, 2017, pp. 1273--1282.

\bibitem{altan_g_SINGH2022380}
S.~Singh, S.~Rathore, O.~Alfarraj, A.~Tolba, B.~Yoon,
  \href{https://www.sciencedirect.com/science/article/pii/S0167739X21004726}{A
  framework for privacy-preservation of iot healthcare data using federated
  learning and blockchain technology}, Future Generation Computer Systems 129
  (2022) 380--388.
\newblock \href {https://doi.org/https://doi.org/10.1016/j.future.2021.11.028}
  {\path{doi:https://doi.org/10.1016/j.future.2021.11.028}}.
\newline\urlprefix\url{https://www.sciencedirect.com/science/article/pii/S0167739X21004726}

\bibitem{altan_h_zheng2022applications}
Z.~Zheng, Y.~Zhou, Y.~Sun, Z.~Wang, B.~Liu, K.~Li, Applications of federated
  learning in smart cities: recent advances, taxonomy, and open challenges,
  Connection Science 34~(1) (2022) 1--28.

\bibitem{altan_i_gadekallu2021federated}
T.~R. Gadekallu, Q.-V. Pham, T.~Huynh-The, S.~Bhattacharya, P.~K.~R.
  Maddikunta, M.~Liyanage, Federated learning for big data: A survey on
  opportunities, applications, and future directions, arXiv preprint
  arXiv:2110.04160 (2021).

\bibitem{konevcny2016federated}
J.~Kone{\v{c}}n{\`y}, H.~B. McMahan, D.~Ramage, P.~Richt{\'a}rik, Federated
  optimization: Distributed machine learning for on-device intelligence, arXiv
  preprint arXiv:1610.02527 (2016).

\bibitem{altan_j_9174890}
F.~Sattler, K.-R. Müller, W.~Samek, Clustered federated learning:
  Model-agnostic distributed multitask optimization under privacy constraints,
  IEEE Transactions on Neural Networks and Learning Systems 32~(8) (2021)
  3710--3722.
\newblock \href {https://doi.org/10.1109/TNNLS.2020.3015958}
  {\path{doi:10.1109/TNNLS.2020.3015958}}.

\bibitem{altan_k_MA2022244}
X.~Ma, J.~Zhu, Z.~Lin, S.~Chen, Y.~Qin,
  \href{https://www.sciencedirect.com/science/article/pii/S0167739X22001686}{A
  state-of-the-art survey on solving non-iid data in federated learning},
  Future Generation Computer Systems 135 (2022) 244--258.
\newblock \href {https://doi.org/https://doi.org/10.1016/j.future.2022.05.003}
  {\path{doi:https://doi.org/10.1016/j.future.2022.05.003}}.
\newline\urlprefix\url{https://www.sciencedirect.com/science/article/pii/S0167739X22001686}

\bibitem{altan_l_9084352}
T.~Li, A.~K. Sahu, A.~Talwalkar, V.~Smith, Federated learning: Challenges,
  methods, and future directions, IEEE Signal Processing Magazine 37~(3) (2020)
  50--60.
\newblock \href {https://doi.org/10.1109/MSP.2020.2975749}
  {\path{doi:10.1109/MSP.2020.2975749}}.

\bibitem{altan_m_MOTHUKURI2021619}
V.~Mothukuri, R.~M. Parizi, S.~Pouriyeh, Y.~Huang, A.~Dehghantanha,
  G.~Srivastava,
  \href{https://www.sciencedirect.com/science/article/pii/S0167739X20329848}{A
  survey on security and privacy of federated learning}, Future Generation
  Computer Systems 115 (2021) 619--640.
\newblock \href {https://doi.org/https://doi.org/10.1016/j.future.2020.10.007}
  {\path{doi:https://doi.org/10.1016/j.future.2020.10.007}}.
\newline\urlprefix\url{https://www.sciencedirect.com/science/article/pii/S0167739X20329848}

\bibitem{tolpegin2020data}
V.~Tolpegin, S.~Truex, M.~E. Gursoy, L.~Liu, Data poisoning attacks against
  federated learning systems, in: European Symposium on Research in Computer
  Security, Springer, 2020, pp. 480--501.

\bibitem{duan2019astraea}
M.~Duan, D.~Liu, X.~Chen, Y.~Tan, J.~Ren, L.~Qiao, L.~Liang, Astraea:
  Self-balancing federated learning for improving classification accuracy of
  mobile deep learning applications, in: 2019 IEEE 37th International
  Conference on Computer Design (ICCD), IEEE, 2019, pp. 246--254.

\bibitem{NEURIPS2020_4ebd440d}
R.~Pathak, M.~J. Wainwright,
  \href{https://proceedings.neurips.cc/paper/2020/file/4ebd440d99504722d80de606ea8507da-Paper.pdf}{Fedsplit:
  an algorithmic framework for fast federated optimization}, in: H.~Larochelle,
  M.~Ranzato, R.~Hadsell, M.~F. Balcan, H.~Lin (Eds.), Advances in Neural
  Information Processing Systems, Vol.~33, Curran Associates, Inc., 2020, pp.
  7057--7066.
\newline\urlprefix\url{https://proceedings.neurips.cc/paper/2020/file/4ebd440d99504722d80de606ea8507da-Paper.pdf}

\bibitem{yuan2020federated}
H.~Yuan, T.~Ma, Federated accelerated stochastic gradient descent, Advances in
  Neural Information Processing Systems 33 (2020).

\bibitem{bhagoji2019analyzing}
A.~N. Bhagoji, S.~Chakraborty, P.~Mittal, S.~Calo, Analyzing federated learning
  through an adversarial lens, in: International Conference on Machine
  Learning, PMLR, 2019, pp. 634--643.

\bibitem{sattler2020byzantine}
F.~Sattler, K.-R. M{\"u}ller, T.~Wiegand, W.~Samek, On the byzantine robustness
  of clustered federated learning, in: ICASSP 2020-2020 IEEE International
  Conference on Acoustics, Speech and Signal Processing (ICASSP), IEEE, 2020,
  pp. 8861--8865.

\bibitem{fung2018mitigating}
C.~Fung, C.~J. Yoon, I.~Beschastnikh, Mitigating sybils in federated learning
  poisoning, arXiv preprint arXiv:1808.04866 (2018).

\bibitem{shen2016auror}
S.~Shen, S.~Tople, P.~Saxena, Auror: Defending against poisoning attacks in
  collaborative deep learning systems, in: Proceedings of the 32nd Annual
  Conference on Computer Security Applications, 2016, pp. 508--519.

\bibitem{blanchard2017machine}
P.~Blanchard, E.~M. El~Mhamdi, R.~Guerraoui, J.~Stainer, Machine learning with
  adversaries: Byzantine tolerant gradient descent, in: Proceedings of the 31st
  International Conference on Neural Information Processing Systems, 2017, pp.
  118--128.

\bibitem{pmlr-v80-mhamdi18a}
E.~M. El~Mhamdi, R.~Guerraoui, S.~Rouault, The hidden vulnerability of
  distributed learning in {B}yzantium, in: Proceedings of the 35th
  International Conference on Machine Learning, Vol.~80 of Proceedings of
  Machine Learning Research, PMLR, 2018, pp. 3521--3530.

\bibitem{bagdasaryan2020backdoor}
E.~Bagdasaryan, A.~Veit, Y.~Hua, D.~Estrin, V.~Shmatikov, How to backdoor
  federated learning, in: International Conference on Artificial Intelligence
  and Statistics, PMLR, 2020, pp. 2938--2948.

\bibitem{kairouz2019advances}
P.~Kairouz, H.~B. McMahan, B.~Avent, A.~Bellet, M.~Bennis, A.~N. Bhagoji,
  K.~Bonawitz, Z.~Charles, G.~Cormode, R.~Cummings, et~al., Advances and open
  problems in federated learning, arXiv preprint arXiv:1912.04977 (2019).

\bibitem{NEURIPS2020_b8ffa41d_yesyoucan}
H.~Wang, K.~Sreenivasan, S.~Rajput, H.~Vishwakarma, S.~Agarwal, J.-y. Sohn,
  K.~Lee, D.~Papailiopoulos,
  \href{https://proceedings.neurips.cc/paper/2020/file/b8ffa41d4e492f0fad2f13e29e1762eb-Paper.pdf}{Attack
  of the tails: Yes, you really can backdoor federated learning}, in:
  H.~Larochelle, M.~Ranzato, R.~Hadsell, M.~F. Balcan, H.~Lin (Eds.), Advances
  in Neural Information Processing Systems, Vol.~33, Curran Associates, Inc.,
  2020, pp. 16070--16084.
\newline\urlprefix\url{https://proceedings.neurips.cc/paper/2020/file/b8ffa41d4e492f0fad2f13e29e1762eb-Paper.pdf}

\bibitem{chen2018draco}
L.~Chen, H.~Wang, Z.~Charles, D.~Papailiopoulos, Draco: Byzantine-resilient
  distributed training via redundant gradients, in: International Conference on
  Machine Learning, PMLR, 2018, pp. 903--912.

\bibitem{xie2019zeno}
C.~Xie, S.~Koyejo, I.~Gupta, Zeno: Distributed stochastic gradient descent with
  suspicion-based fault-tolerance, in: International Conference on Machine
  Learning, PMLR, 2019, pp. 6893--6901.

\bibitem{xie2020zeno++}
C.~Xie, S.~Koyejo, I.~Gupta, Zeno++: Robust fully asynchronous sgd, in:
  International Conference on Machine Learning, PMLR, 2020, pp. 10495--10503.

\bibitem{yin2018byzantine}
D.~Yin, Y.~Chen, R.~Kannan, P.~Bartlett, Byzantine-robust distributed learning:
  Towards optimal statistical rates, in: International Conference on Machine
  Learning, PMLR, 2018, pp. 5650--5659.

\bibitem{altan_n_zhang2020blockchain}
J.~Zhang, S.~Zhong, T.~Wang, H.-C. Chao, J.~Wang, Blockchain-based systems and
  applications: a survey, Journal of Internet Technology 21~(1) (2020) 1--14.

\bibitem{altan_o_9134967}
Y.~Qu, S.~R. Pokhrel, S.~Garg, L.~Gao, Y.~Xiang, A blockchained federated
  learning framework for cognitive computing in industry 4.0 networks, IEEE
  Transactions on Industrial Informatics 17~(4) (2021) 2964--2973.
\newblock \href {https://doi.org/10.1109/TII.2020.3007817}
  {\path{doi:10.1109/TII.2020.3007817}}.

\bibitem{NEURIPS2019_ec1c5914}
G.~Baruch, M.~Baruch, Y.~Goldberg,
  \href{https://proceedings.neurips.cc/paper/2019/file/ec1c59141046cd1866bbbcdfb6ae31d4-Paper.pdf}{A
  little is enough: Circumventing defenses for distributed learning}, in:
  H.~Wallach, H.~Larochelle, A.~Beygelzimer, F.~d\textquotesingle
  Alch\'{e}-Buc, E.~Fox, R.~Garnett (Eds.), Advances in Neural Information
  Processing Systems, Vol.~32, Curran Associates, Inc., 2019.
\newline\urlprefix\url{https://proceedings.neurips.cc/paper/2019/file/ec1c59141046cd1866bbbcdfb6ae31d4-Paper.pdf}

\bibitem{NEURIPS2019_415185ea}
S.~Rajput, H.~Wang, Z.~Charles, D.~Papailiopoulos,
  \href{https://proceedings.neurips.cc/paper/2019/file/415185ea244ea2b2bedeb0449b926802-Paper.pdf}{Detox:
  A redundancy-based framework for faster and more robust gradient
  aggregation}, in: H.~Wallach, H.~Larochelle, A.~Beygelzimer,
  F.~d\textquotesingle Alch\'{e}-Buc, E.~Fox, R.~Garnett (Eds.), Advances in
  Neural Information Processing Systems, Vol.~32, Curran Associates, Inc.,
  2019.
\newline\urlprefix\url{https://proceedings.neurips.cc/paper/2019/file/415185ea244ea2b2bedeb0449b926802-Paper.pdf}

\bibitem{Li2020On}
X.~Li, K.~Huang, W.~Yang, S.~Wang, Z.~Zhang,
  \href{https://openreview.net/forum?id=HJxNAnVtDS}{On the convergence of
  fedavg on non-iid data}, in: International Conference on Learning
  Representations, 2020.
\newline\urlprefix\url{https://openreview.net/forum?id=HJxNAnVtDS}

\bibitem{lecun1998mnist}
Y.~LeCun, The mnist database of handwritten digits, http://yann. lecun.
  com/exdb/mnist/ (1998).

\bibitem{krizhevsky2009learning}
A.~Krizhevsky, G.~Hinton, et~al., Learning multiple layers of features from
  tiny images (2009).

\bibitem{xiao2017fashion}
H.~Xiao, K.~Rasul, R.~Vollgraf, Fashion-mnist: a novel image dataset for
  benchmarking machine learning algorithms, arXiv preprint arXiv:1708.07747
  (2017).

\end{thebibliography}





\end{document}